%% file: main.tex
\documentclass{article}

\usepackage[accepted]{icml2023}

\usepackage{microtype}
\usepackage{graphicx}
\usepackage{booktabs} % for professional tables

\usepackage{amsmath}
\usepackage{amssymb}
\usepackage{mathtools}
\usepackage{amsthm}
\usepackage{bm}
\usepackage{thmtools}
\usepackage{thm-restate}
\usepackage{subfiles}
\usepackage{subcaption}
\usepackage{caption}
\newcommand{\e}{\mathrm{e}}
\renewcommand{\d}{\mathrm{d}}
\newcommand{\E}{\mathbb{E}}
\newcommand{\V}{\mathbb{V}}
\hypersetup{
    colorlinks,
    linkcolor={red!50!black},
    citecolor={blue!50!black},
    urlcolor={blue!80!black}
}
\theoremstyle{plain}
\newtheorem{theorem}{Theorem}[section]
\newtheorem{proposition}[theorem]{Proposition}
\newtheorem{lemma}[theorem]{Lemma}
\newtheorem{corollary}[theorem]{Corollary}
\newtheorem*{theoremr*}{Theorem~\ref{th1}}
\newtheorem*{theoremr2*}{Theorem~\ref{th2}}
\newtheorem*{corr*}{Corollary~\ref{cor1}}
\newtheorem*{propositionr*}{Proposition~\ref{prop1}}
\newtheorem*{propositionr2*}{Proposition~\ref{p2}}
\theoremstyle{definition}
\newtheorem{definition}[theorem]{Definition}
\newtheorem{assumption}[theorem]{Assumption}
\theoremstyle{remark}
\newtheorem{remark}[theorem]{Remark}

\usepackage[textsize=tiny]{todonotes}

\icmltitlerunning{Controlling Posterior Collapse by an Inverse Lipschitz Constraint on the Decoder Network}

\begin{document}

\twocolumn[\icmltitle{Controlling Posterior Collapse by an Inverse Lipschitz Constraint on the Decoder Network}




\begin{icmlauthorlist}
\icmlauthor{Yuri Kinoshita}{utokyo}
\icmlauthor{Kenta Oono}{pfn}
\icmlauthor{Kenji Fukumizu}{ist}
\icmlauthor{Yuichi Yoshida}{nii}
\icmlauthor{Shin-ichi Maeda}{pfn}
\end{icmlauthorlist}

\icmlaffiliation{utokyo}{The University of Tokyo, Tokyo, Japan. Work done at Preferred Networks.}
\icmlaffiliation{pfn}{Preferred Networks, Inc., Tokyo, Japan}
\icmlaffiliation{ist}{The Institute of Statistical Mathematics, Tokyo, Japan}
\icmlaffiliation{nii}{National Institute of Informatics (NII), Tokyo, Japan}

\icmlcorrespondingauthor{Yuri Kinoshita}{yuri-kinoshita111@g.ecc.u-tokyo.ac.jp}
\icmlcorrespondingauthor{Shin-ichi Maeda}{ichi@preferred.jp}

\icmlkeywords{Machine Learning, ICML}

\vskip 0.3in
]



\printAffiliationsAndNotice{}  % leave blank if no need to mention equal contribution

\begin{abstract}
Variational autoencoders (VAEs) are one of the deep generative models that have experienced enormous success over the past decades. However, in practice, they suffer from a problem called posterior collapse, which occurs when the encoder coincides, or collapses, with the prior taking no information from the latent structure of the input data into consideration. In this work, we introduce an inverse Lipschitz neural network into the decoder and, based on this architecture, provide a new method that can control in a simple and clear manner the degree of posterior collapse for a wide range of VAE models equipped with a concrete theoretical guarantee. We also illustrate the effectiveness of our method through several numerical experiments.
\end{abstract}

\section{Introduction}
\label{sec1}
\subfile{sec1.tex}
\section{Preliminaries}
\label{sec2}
\subfile{sec2.tex}
\section{Theoretical Analysis}
\label{sec3}
\subfile{sec3.tex}
\section{Implementation}
\label{sec4}
\subfile{sec4.tex}
\section{Experiments}
\label{sec5}
\subfile{sec5.tex}
\section{Conclusion}
\label{sec6}
\subfile{sec6.tex}

\section*{Acknowledgements}

We thank anonymous reviewers for their valuable feedback and advice. KS has been supported in part by JST CREST JPMJCR2015 and JSPS Grant-in-Aid for Transformative Research Areas (A) 22H05106.

\bibliography{biblio}
\bibliographystyle{icml2023}

\newpage
\appendix
\onecolumn
\section{Proof of Theorems and Propositions}
\label{apa}
\subfile{apa.tex}
%\section{Lower bound of Relative Fisher Information Divergence for the General Case}
%\label{apd}
%\subfile{apd.tex}
\section{Details of Experiments}
\label{apb}
\subfile{apb.tex}
\section{Log-Sobolev Inequality}
\label{apc}
\subfile{apc.tex}
\end{document}

%% file: sec1.tex
\subsection{Background and Organization}
Over the past decades, generative models that aim to capture the distribution of a given data have intensively contributed to the creation of many performant algorithms in the field of machine learning and artificial intelligence. The recent surge of interest in incorporating expressive neural networks into statistical and probabilistic methods has even more enhanced their ability to handle high-dimensional data such as image, text and speech. Especially, variational autoencoders (VAEs) are one of these deep generative models that have experienced enormous success~\citep{KW2013,RMW2014}. They can draw low-dimensional latent random variables from a predefined prior distribution and transform them into meaningful data using deep neural networks trained on a tractable objective function called the evidence lower bound (ELBO).
\par While VAEs are nowadays omnipresent in the field of machine learning, it is also widely recognized that there remain in practice some major challenges that still require effective solutions. Notably, they suffer from the problem of \emph{posterior collapse}, which occurs when the distribution corresponding to the encoder coincides, or collapses, with the prior taking no information from the latent structure of the input data into consideration. Also known as \emph{KL vanishing} or \emph{over-pruning}, this phenomenon makes VAEs incapable to produce pertinent representations and has been reportedly observed in many fields (e.g.,~\citet{BVVDJB2016,FLLGCC2019,WZ2022,YK2017}). There exists now a large body of literature that examines its underlying causes and presents various techniques to prevent it (e.g.,~\citet{BVVDJB2016,HS2019,RO2019}). Please refer to Subsection \ref{rw} for further details. 
\par Despite this abundant number of studies conducted so far, the mechanism of posterior collapse is not completely understood, and many different approaches, such as $\beta$-VAE and $\delta$-VAE, have been suggested over time based on varied hypotheses and theories. Some blame the variational inference~\citep{BGS2015,BVVDJB2016,CKSDDSSA2016,FLLGCC2019,HTLC2018,HT2020,SRMSW2016,ZLE2018}, some focus on the optimization procedure~\citep{HS2019,KW2018,LH2019}, and others hold the formulation of the model responsible~\citep{DWW2020,GK2016,YH2017,VV2017,ZY2020,DK2019,YK2017,RO2019}. Nevertheless, most proposed methods are based on heuristics and crucially lack convincing theoretical guarantees. That is why a line of work also tries to rigorously analyse the mechanism of posterior collapse~\citep{DWW2020, WZ2022,LT2019}. For example, the recent work of~\citet{WBC2021} showed that posterior collapse and latent variable non-identifiability are equivalent. These theoretical works have indeed helped us recognize the primary cause of this phenomenon. However, they require the explicit formulation of the VAE and objective function or too rigid a definition of posterior collapse. This means guarantees and suggested techniques are either only applicable to simple problems or lack practical usefulness.
\par Therefore, a technique that has both a theoretical guarantee and a broad applicable spectrum is first and foremost required. In this paper, we investigate a method that can control in a simple and clear manner the degree of posterior collapse for a wide range of VAE models equipped with a concrete analysis that assures this control.
\subsection{Contributions}
The major contributions of this paper can be summarized as follows:
\begin{itemize}
    \item We introduce the concept of \emph{inverse Lipschitzness} into the underlying decoder and prove under minor assumptions that the degree of posterior collapse can be controlled by this property. 
    \item Based on this theoretical guarantee, we provide the first method that can not only directly adjust the degree of posterior collapse but is also simple and applicable to a broad type of models. 
    \item We explain and illustrate with experiments that our method is effective and can outperform prior works.
\end{itemize}
\subsection{Related Works}\label{rw}
The first held responsible for posterior collapse was the Kullback-Leibler (KL) divergence between the encoder and prior, present in the formulation of the ELBO~\citep{BVVDJB2016}. It would force the model to prioritize its minimization, leading to posterior collapse. As a result, many previous works have tried to attenuate its influence during the training with an annealing scheme~\citep{HTLC2018,SRMSW2016,FLLGCC2019}. These works are often summarized as $\beta$-VAE originally created for other goals~\citep{HM2016}. More loosely, the problem arises in a sense because optimization fails and is driven into undesired minima. This has encouraged some researchers to find tighter bounds than the ELBO~\citep{BGS2015} or other objective functions~\citep{ZLE2018,CKSDDSSA2016,HT2020}. There have even been attempts to find a more suited optimization procedure~\citep{HS2019,KW2018,LH2019}. Others have pointed out that this phenomenon mainly happens when VAEs involve high flexibility due to neural networks~\citep{DWW2020}. This point of view has incited researchers to restrict the flexibility of VAEs~\citep{GK2016,YH2017} or to modify their architecture~\citep{VV2017,ZY2020,DK2019,YK2017,LW2022}. Although posterior collapse can be regarded as the optimization falling into local minima, the recent theoretical investigation of~\citet{LT2019} underlined that these spurious stationary points are not created by the ELBO but are actually inherent in the exact maximization of the marginal log-likelihood. 
\par $\delta$-VAE~\citep{RO2019} is a model that constrains the variational family of the posterior to assure a minimum distance from the prior in terms of KL-divergence. That way, they avoid, by definition, posterior collapse. However, finding parameters that satisfy this structural constraint involves an additional tedious optimization in general. In this paper, we will provide a method that can similarly control the discrepancy between the posterior and prior but without any calculations needed at all.
\par As the most relevant work to this paper,~\citet{WBC2021} showed that posterior collapse occurs if and only if latent variables are non-identifiable in the generative model. They proposed an architecture called latent identifiable VAE (LIDVAE) that solves this problem of non-identifiability without losing the flexibility of VAEs. Their definition of posterior collapse requires the posterior and prior to be exactly equal, which means their model cannot avoid the usual case in practice where the posterior and prior are \textit{nearly} equal. In this paper, we will work in the same framework as~\citet{WBC2021} but provide a simple model that can handle broader situations closer to reality.
\par In the context of generative adversarial networks,~\citet{YK2019} introduced into the transport map a concept equivalent to the inverse Lipschitzness in order to promote the entropy of the generator distribution and to avoid a phenomenon called mode collapse. In contrast to inverse Lipschitzness, the role of Lipschitz continuity in VAEs and other machine learning methods has been the subject of much research~\citep{YR2020,BC2022}. For example,~\citet{BC2022} constrained many components of the VAE including the encoder and the decoder to satisfy Lipschitz continuity and showed this to be useful to increase VAE robustness against adversarial attacks.
\paragraph{Organization} In Section~\ref{sec2}, we will first explain the basic formulation of VAEs and clarify necessary mathematical backgrounds. Section~\ref{sec3} will be devoted to the description of our theoretical analysis, and Section~\ref{sec4} to the implementation of our proposed model. Finally, Section~\ref{sec5} will illustrate its effectiveness with synthetic and real-world data.
\paragraph{Notation} The Euclidean norm is denoted by $\|\cdot\|$ for vectors. The exponential family $h(x)\exp\left\{T(x)^\top \xi-A(\xi)\right\}$, where $T(x)=(T_1(x),\ldots,T_t(x))^\top$ is called the sufficient statistic, is represented as $\mathrm{EF}_{T}(x\mid \xi)$, abbreviating the dependence on $h$ since it does not appear in our analysis. 

%% file: sec2.tex
In this section, we briefly explain the mathematical background for VAEs and posterior collapse.
\subsection{Variational Autoencoders}
Let $\bm{x}=(x_1,\ldots, x_n)\in \mathbb{R}^{m\times n}$ be $n$ i.i.d.\ data points in $\mathbb{R}^m$, and suppose they were generated from an underlying structure. In other words, VAEs assume there exists a latent variable $z\in \mathbb{R}^l$ sampled from a pre-defined prior $p(z)$ that creates the observed data through a conditional distribution $p_\theta(x\mid z)$ parameterized over $\theta$, also called a generative model. In short,
\begin{align}
    z_i\sim p(z),\quad x_i\sim p_\theta(x\mid z_i)\quad \forall i=1,\ldots,n. \label{model}
\end{align}
Under this problem setting, the ultimate goal of VAE is to maximize the following marginal log-likelihood:
\[
\log p_\theta(\bm{x})=\sum_{i=1}^n \log \int p_\theta(x_i\mid z)p(z)\d z.
\]
\par While this optimization can in theory determine values of $\theta$, it requires some intractable computations involving the marginalization over $z$. Therefore, a tractable lower bound called evidence lower bound (ELBO) has been proposed to avoid this problem. The idea is to first introduce a recognition model $q_\phi(z\mid x)$, which will approximate the true posterior $p_\theta(z\mid x)$. Then, the marginal log-likelihood can be formulated as follows:
\begin{align*}
\log p_\theta(\bm{x})=\sum_{i=1}^n\mathcal{L}_{\theta,\phi}(x_i)+D(q_\phi(z\mid x_i)\mid\mid p_\theta(z\mid x_i)),
\end{align*}
where 
\begin{align}
    \mathcal{L}_{\theta,\phi}(x)\vcentcolon=&\E_{q_\phi(z\mid x)}[\log p_\theta(x\mid z)]-D(q_\phi(z\mid x)\mid\mid p(z)),\label{klterm}
\end{align}
and $D(\cdot\mid\mid \cdot)$ refers to the KL-divergence. Consequently, we obtain the ELBO defined as the right-hand side of the following inequality:
\begin{align*}
\log p_\theta(\bm{x})\ge \sum_{i=1}^n\mathcal{L}_{\theta, \phi}(x_i).
\end{align*}
Here, we used the fact that KL-divergence is non-negative, and the equality holds only if $D(q_\phi(z\mid x_i)\mid\mid p_\theta(z\mid x_i))=0$ for all $x_i$. Via a reparameterization trick, the optimization over the ELBO in terms of $\theta$ and $\phi$ becomes more tractable than that of the log-likelihood~\citep{KW2013}. In this context, the recognition model $q_\phi(z\mid x)$ is often called an encoder, and the generative model $p_\theta(x\mid z)$ a decoder.
\subsection{Posterior Collapse}
Posterior collapse refers to the situation when the encoder coincides, or collapses, with the prior taking no information from the latent structure of the input data into consideration. Mathematically, this can be translated into the following general definition.
\begin{definition}[$\epsilon$-posterior collapse]\label{epc}
 For a given parameter $\hat\phi$, a data set $\bm{x}=(x_1,\ldots, x_n)\in \mathbb{R}^{m\times n}$ and a closeness criterion $d(\cdot,\cdot)$, an {\em $\epsilon$-posterior collapse} is defined for a given $\epsilon\ge 0$ by the condition
\begin{align}
     d(q_{\hat\phi}(z\mid x_i),p(z))\le \epsilon\quad \forall i=1,\ldots,n.\label{epceq}
\end{align} 
\end{definition}
For example, \citet{RO2019} concentrated on the case where the closeness criterion is
KL-divergence. The problem setting of~\citet{WBC2021} corresponds to a particular situation with $\epsilon$ set to 0 (see Definition \ref{defpc}). In fact, theoretical analyses that can treat such a level of abstract definition as $\epsilon$-posterior collapse are quite rare since most use the explicit formulation of the model, ELBO and marginal log-likelihood in order to draw conclusions. An even broader definition was introduced by~\citet{LT2019} for measurement purpose named $(\epsilon,\delta)$-posterior collapse which substitutes equation \eqref{epceq} with the stochastic formula $\Pr_x\left(d(p_{\hat\theta}(z\mid x_i),p(z))\le\epsilon\right)>1-\delta$. This formulation is outside the scope of this work, while an extension to this case is an interesting future direction.
\par Many elements of a VAE are held responsible for this phenomenon, namely the KL term in equation \eqref{klterm}, the variational approximation, the optimization scheme, and the model itself. Particularly, \citet{WBC2021} showed that the $0$-posterior collapse of Definition~\ref{defpc} and latent variable non-identifiability (Definition~\ref{defni}) are equivalent. This equivalence implies that it is sufficient to make the likelihood function $p_{\theta}(x\mid z)$ injective in terms of $z$ for all parameters $\theta$ in order to avoid posterior collapse of Definition~\ref{defpc}. This led to the model called LIDVAE proposed by~\citet{WBC2021}.
\begin{definition}[posterior collapse,~\citet{WBC2021}]\label{defpc}
 For a given parameter $\hat\theta$ and data set $\bm{x}=(x_1,\ldots, x_n)\in \mathbb{R}^{m\times n}$, {\em posterior collapse} occurs if $p_{\hat\theta}(z\mid \bm{x})=p(z)$.
\end{definition}
\begin{definition}[latent variable non-identifiability,~\citet{WBC2021}]\label{defni}
 For a given parameter $\hat\theta$ and data set $\bm{x}=(x_1,\ldots, x_n)\in \mathbb{R}^{m\times n}$, the latent variable is {\em non-identifiable} if $p_{\hat{\theta}}(\bm{x}\mid z)=p_{\hat{\theta}}(\bm{x}\mid z')$ for all $z$, $z'$,
 i.e., the likelihood of the data set $\bm{x}$ does not depend on the latent variable.
\end{definition}
Note they assumed that the variational approximation is exact. That is, the encoder can represent the posterior $p_\theta(z\mid x)$. This assumption is sensible; if the exact inference already presents symptoms of posterior collapse, this is the first problem to tackle before concentrating on the approximated case, which can only aggravate the situation. In the remainder of this paper, we will treat this exact case as well. In other words, we will investigate methods to mitigate the $\epsilon$-posterior collapse inherent in the formulation of the model and not that caused by any sort of approximation.
\par However, their definition of posterior collapse is clearly too restrictive because it requires the posterior and prior to be exactly equal, which means their model cannot avoid the usual case in practice where these are \textit{nearly} equal. This motivates us to extend Definition~\ref{defpc} to the general one (Definition~\ref{epc}) and discuss how to prevent it. 

%% file: sec3.tex
In this section, we will show that the general posterior collapse of Definition \ref{epc} in terms of the relative Fisher information divergence can be controlled by a simple inverse Lipschitz constraint on the decoder network.
\subsection{Assumptions and Problem Setting}
Let us first clarify our problem setting.
\subsubsection{Generative Model}
\par We introduce the concept of inverse Lipschitzness.
\begin{definition}[inverse Lipschitzness]
    Let $L\ge 0$. $f:\mathbb{R}^l\to \mathbb{R}^t$ is $L$-inverse Lipschitz if $\|f(x)-f(y)\|\ge L\|x-y\|$ holds for all $x,y \in \mathbb{R}^l$.
\end{definition}
If $f$ is $L$-inverse Lipschitz with $L>0$, then it is injective. Therefore, the restriction of $f$ on its image possesses an inverse which is $1/L$-Lipschitz. Inverse Lipschitz functions can be regarded as a stronger condition than injectivity, which is the property of the decoder network suggested by~\citet{WBC2021} in order to avoid latent variable non-identifiability, and consequently posterior collapse. The motivation to introduce this stronger concept is to capture more nuances in the latent variable identifiability with the inverse Lipschitz constant $L$, which is not possible with simple injective functions.
\paragraph{Construction of inverse Lipschitz functions} Inverse Lipschitz neural networks can be generated by Brenier maps~\citep{B2004,WBC2021}. A \emph{Brenier map} is a function that is the gradient of a real-valued convex function. The gradient of a real-valued $L$-strongly convex function $F$ (i.e., $F(x)-L\|x\|^2/2$ is convex) becomes $L$-inverse Lipschitz. Therefore, theoretical results will be proved for this type of inverse Lipschitz functions, derivatives of strongly convex functions. Notably, this means we can only handle functions with the same input dimension and output dimension. Please refer to Section~\ref{sec4} for further details. 
\par We can now state our main assumption.
\begin{assumption}\label{as1}
The generative model $p_\theta(x\mid z)$ is an exponential family so that $p_\theta(x\mid z)=\mathrm{EF}_{T}(x\mid f_\theta(z))$, where $f_\theta:\mathbb{R}^l\to \mathbb{R}^t$ is constructed as follows. If $l=t$, $f_\theta$ is an $L$-inverse Lipschitz function generated from a Brenier map, and we denote $\Theta_L$ as the set of parameters $\theta$ that achieve this property. If $l< t$, $f_\theta=f_\theta^{(2)}(B^\top f_\theta^{(1)}(z))$, where $f_\theta^{(1)}:\mathbb{R}^l\to \mathbb{R}^l$ and $f_\theta^{(2)}:\mathbb{R}^t\to \mathbb{R}^t$ are respectively inverse Lipschitz with constant $L_1$ and $L_2$ generated from Brenier maps. $B$ is a $t\times l$ diagonal matrix with all diagonal elements with value 1. Likewise, we define the set $\Theta_{L_1,L_2}$.
\end{assumption}
This will be the sole condition that we will impose on the model. We restrict neither the type of prior nor the optimization scheme. The assumption that the generative model or likelihood function is an exponential family is keeping large liberty to the model. It is even less restrictive than some prior works on posterior collapse, which often necessitate all components, including the generative model, to be Gaussian (e.g.,~\citet{DWW2020,LT2019}). The sole limitation is our requirement of inverse Lipschitzness, which may restrict the expression of the likelihood function. However, the effect of this restriction is precisely one of the interests of this paper. Moreover, it is shown that when $f_\theta$ is only injective, $\mathrm{EF}_{T}(x\mid f_\theta(z))$ can model any distributions of the form $\mathrm{EF}_{T}(x\mid f(z))$ where $f$ is an arbitrary function~\citep{WBC2021}. Therefore, by adjusting this inverse Lipschitz constant, we can cover the full spectrum, i.e., from the case with no restriction ($L\to 0$) to the extremely restrictive ($L\to\infty$), which implies our problem setting still leaves considerable freedom to the model.
\subsubsection{Criterion}
\par As a closeness criterion, we will select the relative Fisher information divergence.
\begin{definition}
    We define the {\em relative Fisher information divergence} $F(\cdot\mid\mid\cdot)$ of $p(x)$ with respect to $q(x)$ as 
    \[
    F(p(x)\mid\mid q(x))\vcentcolon =\int \|\nabla \log p(x)-\nabla \log q(x)\|^2p(x)\d x.
    \]
\end{definition}
This divergence is used for a wide range of statistical analysis and machine learning applications \citep{YM2019,OV2000,EH2021,HW2017,W2016,HC2018}. It is also intrinsically related to the Hyv\"{a}rinen score~\citep{HD2005}. Furthermore, many types of distributions, such as Gaussian and Gaussian mixture, satisfy the log-Sobolev inequality (LSI), which provides an upper bound of the KL-divergence in terms of the relative Fisher information divergence (see Appendix \ref{apc} for the exact formulation). Many distributions can satisfy this inequality since LSI is robust to bounded perturbation and Lipschitz mapping~\citep{G1975,HS1986,L1999}. As a consequence, if a posterior that satisfies LSI collapses on the prior in terms of relative Fisher divergence, so will it in terms of KL-divergence. On the other hand, de Bruijn's identity relates KL divergence to relative Fisher divergence. Under some additional conditions, we can show that the control of the latter results in that of the former (see Proposition \ref{prop1} for further details). Hence, our choice of discrepancy is sensible since it is necessary and sometimes even sufficient to avoid the collapse in terms of relative Fisher divergence in order to prevent that in terms of KL-divergence.
\subsection{Theoretical Guarantee} 
We are now ready to state our main theorem, which shows that under Assumption~\ref{as1}, posterior collapse can be efficiently controlled by the inverse Lipschitz constant. We will first discuss in detail the case where $l=t$ for clarity, and then show that similar statements hold for the general case as well.
\begin{theorem}\label{th1}
Under model~\eqref{model}, Assumption~\ref{as1} and $l=t$, the following holds for all $i$ and $\theta\in\Theta_L$:
\begin{align*}
    & F( p(z)\mid\mid p_\theta(z\mid x_i)) \nonumber\\
    & \quad \ge L^2\int\|T(x_i)-\E_{p_\theta(x\mid z)}[T(x)]\|^2 p(z)\d z.
\end{align*}
\end{theorem}
See Appendix~\ref{sec:proof_main1} for the proof.
\begin{remark}
The crux of this theorem is the relation 
\[
\|\nabla_z \log p_\theta(z\mid x)-\nabla_z \log p(z)\|=\|\nabla_z\log p_\theta(x\mid z)\|.
\]
This equation relates posterior collapse (left-hand side) to the latent variable non-identifiability (right-hand side). Indeed, if the non-identifiability in Definition \ref{defni} holds, then $p_\theta(x\mid z)$ would be constant with respect to $z$, leading to a zero derivative and posterior collapse in the strict sense of the term. This use of derivatives enables us to additionally capture nuances of the latent variable identifiability and essentially contributes to our main result. A similar bound of Theorem~\ref{th1} can thus be derived for $F(p_\theta(z\mid x_i)\mid\mid p(z))$ as well.
\end{remark}
\begin{corollary}\label{cor1}
Under model~\eqref{model}, Assumption~\ref{as1} and $l=t$,
\begin{align*}
    &F( p(z)\mid\mid p_\theta(z\mid x_i))\nonumber\\
     &  \quad \ge L^2\inf_{\theta\in\Theta_L}\left\{\int \|T(x_i)-\E_{p_\theta(x\mid z)}[T(x)]\|^2 p(z)\d z\right\}.
    %    &F(p_\theta(z\mid x_i)\mid\mid p(z))\nonumber\\
    %\ge & L^2\inf_{\theta\in\Theta_L}
     %\E_{p(z)} \left[  \|T(x_i)-\E_{p_\theta(x\mid z)}[T(x)]\|^2 \right].
\end{align*}
Therefore, the lower bound is non-decreasing in terms of L. Moreover, if the infinimum of this lower bound is attained by a parameter $\theta\in\Theta_L$ and $p(z)$ has a positive variance, then the lower bound is monotonically increasing in terms of L.
\end{corollary}
In other words, we are guaranteed to avoid an $\epsilon$-posterior collapse as long as we take a sufficiently large $L$. Increasing this value has the effect of moving away the posterior from the prior without any limits. In this sense, we can \textit{control} the degree of posterior collapse only with the inverse Lipschitz constant.
\par This theorem adjusted for an empirical version of the relative Fisher information divergence provides additional insights concerning another formulation of the lower bound.
\begin{theorem}\label{th2}
Under model~\eqref{model}, Assumption~\ref{as1} and $l=t$ the following holds for all $\theta\in\Theta_L$:
\begin{align}\label{eq26}
    &\bar{F}_\theta(\bm{x}) \vcentcolon=\nonumber\\
     &\int \left\|\frac{1}{n}\sum_{i=1}^n\nabla_z \log p_\theta(z\mid x_i)-\nabla_z \log p(z)\right\|^2p(z)\d z
     %\E_{p(z)} \left[ \left\|\frac{1}{n}\sum_{i=1}^n\nabla_z \log p_\theta(z\mid x_i)-\nabla_z \log p(z)\right\|^2 \right]
     \nonumber\\
    &\quad \ge  L^2
     %\E_{p(z)} \left[ \left\|\frac{1}{n}\sum_{i=1}^nT(x_i)-\E_{p_\theta(x\mid z)}[T(x)]\right\|^2 \right]
     \int\left\|\frac{1}{n}\sum_{i=1}^nT(x_i)-\E_{p_\theta(x\mid z)}[T(x)]\right\|^2p(z)\d z.
\end{align}
\end{theorem}
See Appendix~\ref{sec:proof_main2} for the proof. 

\paragraph{Bias-Variance Decomposition} Now, if we have enough samples (i.e., $n\to \infty$), we can approximate $\frac{1}{n}\sum_{i=0}^nT(x_i)$ as the expectation under the true distribution $p^\ast(x)$, $\E_{p^\ast(x)}[T(x)]$. Moreover, consider the generative model contains the true model $\theta^\ast$. This means there exists $\theta^\ast$ such that $p^\ast(x)=\int p_{\theta^\ast}(x\mid z)p(z)\d z$. Let us define $S_\theta\vcentcolon=E_{p_\theta(x\mid z)}[T(x)]$. Then, the integral of the right-hand side of Equation~\eqref{eq26} can be reformulated as 
\begin{align*}
    \int\|\E_{p_\theta(x\mid z)}&[T(x)]-\E_{p^\ast(x)}[T(x)]\|^2p(z)\d z\\
    =&\int\|S_\theta-\E[S_{\theta^\ast}]\|^2p(z)\d z\\
    =& \V[S_\theta]+\|\E[S_{\theta^\ast}]-\E[S_{\theta}]\|^2,
\end{align*} 
where $\V[\cdot]$ is the variance of $S_\theta$ in terms of $z\sim p(z)$.
As a result,
\begin{align*}
     \bar{F}_\theta(\bm{x})\ge& L^2\left(\V[S_\theta]+\|\E[S_{\theta^\ast}]-\E[S_{\theta}]\|^2\right)\label{lb}.
\end{align*}
Interestingly, the lower bound can be written as the sum of the variance of $S_\theta$ and its bias with the true parameter. 
\paragraph{General Case}
Now, let us state a similar theorem that holds for the general case $l< t$ under an additional condition.
\begin{theorem}\label{th3}
Under Assumptions \ref{as1}, $l<t$ and that $\xi \mapsto \nabla_\xi A(\xi)=\E_{\mathrm{EF}_T(x\mid \xi)}[T(x)]$ is a diffeomorphism, $F(p(z)\mid\mid p_\theta(z\mid x)$ is lower-bounded by a term that is increasing in terms of $L_1$.
\end{theorem}
See Appendix~\ref{apa4} for the precise formulation and for the proof. 
\paragraph{Expansion to KL divergence}
Finally, although it may require stronger assumptions, we can derive a lower bound of KL divergence between the posterior and prior from the bound of Fisher divergence as follows. 
\begin{proposition}\label{prop1}
Suppose that the lower bound of Fisher divergence $F(p(x) \mid\mid  q(x)) \geq \varepsilon$ holds for any small perturbations of $p$ and $q$ to some extent. More precisely, let $p_t$ (or $q_t$) denote the convolution between $p$ ($q$, resp.) and $N(0,t)$. Assume that there is $\delta > 0$ such that $F(p_t \mid\mid  q_t) \ge \epsilon$ for any $t\in[0,\delta]$. Then, the bound $D(p \mid\mid  q) \ge \frac{1}{2}\delta \epsilon$ holds.
\end{proposition}
See Appendix~\ref{sec:proof_prop} for the proof. This suggests that a lower bound of the Fisher divergence can also control the lower bound of KL divergence, and thus our method avoids posterior collapse in terms of KL divergence as well.
\subsection{Discussion}
The theoretical analysis led in the previous subsection considerably contributes to the research on posterior collapse in several aspects. First of all, we proved this phenomenon, inherent in the formulation of the model, can be controlled by the inverse Lipschitz constant of the underlying decoder. Indeed, increasing this constant forces the discrepancy between the prior and posterior to become larger. This kind of analysis against the general $\epsilon$-posterior collapse was provided neither in the work of~\citet{WBC2021} nor in most previous works at all. In fact, theoretical guarantee is often missing in previously proposed heuristic methods, such as $\beta$-VAE. $\delta$-VAE similarly requires finding parameters of prior and decoder that satisfy a discrepancy constraint in terms of KL-divergence, but this needs some intractable and heavy optimization in general. On the contrary, our method is far easier and simpler since we know in advance what kind of parameter is required: the inverse Lipschitz constant. Finally, while previous analyses only treated marginal or simple cases such as the Gaussian VAE~\citep{LT2019}, it is also important to note that our model can be used for the general exponential family. In short, we provided the first solution based on an inverse Lipschitz constraint that can \textit{control} the degree of posterior collapse equipped with a concrete theoretical guarantee, is simple and is applicable to a broad type of models.

\par On the other hand, our method also presents some drawbacks. The main limitation of our theory is that we only assure to avoid posterior collapse in terms of the relative Fisher information divergence, a weaker concept than that in terms of KL-divergence. This relaxation was necessary to proceed into rigorous analysis and derive theoretical guarantees. Nonetheless, we showed that posterior collapse in terms of KL divergence can also be controlled under some further assumptions and will show in the experiments that our method can often alleviate this stronger posterior collapse as well even though it is outside the scope of our guarantee. Furthermore, while increasing $L$ can avoid posterior collapse, this has, at the same time, the effect of contracting the set $\Theta_L$ at the risk of limiting the flexibility of the model. It is thus essential to find a good balance. However, this trade-off is clear thanks to our analysis, and the tuning is quite simple. 

%% file: sec4.tex
In this section, we will describe the implementation of our model. It turns out that it is rather simple since it only extends the model LIDVAE proposed by~\citet{WBC2021}. We just have to focus on the realization of a neural network $f_\theta$ that is inverse Lipschitz with respect to its output. The idea is to modify the Input Convex Neural Network (ICNN) of~\citet{AX2017} and compute its Brenier map so that it becomes an $L$-inverse Lipschitz function.
\subsection{Brenier Maps and ICNN}
\label{sec:BrenierICNN}
As previously mentioned, a Brenier map is the gradient of a real-valued convex function~\citep{B2004,WBC2021}. It is injective by definition. In addition, that of an $L$-strongly convex function will become $L$-inverse Lipschitz as desired.
\par ICNN is a neural architecture that creates convex functions~\citep{AX2017}. As long as this property is satisfied other algorithms can be chosen, but for clarity, we will use this as an example. \citet{AX2017} defined the fully connected ICNN $G_\theta:\mathbb{R}^l\to \mathbb{R}$ with $k$ layers and input $z\in\mathbb{R}^l$ as follows:
\begin{align*}
    y_{i+1}&=g_i(W_i^{(y)}y_i+W_i^{(z)}z+b_i)\ (i=0,\ldots,k-1),\\
    G_\theta(z)&=y_k,
\end{align*}
where $\{W_i^{(y)}\}_i$ are non-negative, and all functions $g_i$ are convex and non-decreasing. This kind of neural network is not only convex but is also known to be a universal approximator of convex function on a compact domain endowed with the sup norm~\citep{CS2018}. Given an ICNN $G_\theta$, it is thus not difficult to extend it to an $L$-strongly convex function since it suffices to add a regularization term $L\|z\|^2/2$ to the output as follows:
\begin{align*}
    y_{i+1}&=g_i(W_i^{(y)}y_i+W_i^{(z)}z+b_i)\ (i=0,\ldots,k-1),\\
    G_\theta(z)&=y_k+\frac{L}{2}\|z\|^2,
\end{align*}
where $\{W_i^{(y)}\}_i$ are non-negative and all functions $g_i$ are convex and non-decreasing. 
\par Provided an $L$-strongly convex ICNN $G_\theta$, we can compute its gradient as $f_\theta=\d G_\theta/\d z$ and obtain our desired  $L$-inverse Lipschitz neural network.
\subsection{IL-LIDVAE and its Variants}
The construction in Subsection~\ref{sec:BrenierICNN} leads to an extension of LIDVAE, simple but with strong theoretical guarantees, as shown in Section~\ref{sec3}. We will call it Inverse Lipschitz LIDVAE (IL-LIDVAE) in order to distinguish it from LIDVAE and other prior works.
\begin{definition}[IL-LIDVAE] \label{def:IL-LIDVAE}
We define IL-LIDVAE with inverse Lipschitz constants $(L_1,L_2)$ as the following generative model:
\begin{align*}
    z\sim p(z), \ \ \ x\sim \mathrm{EF}(x\mid f^{(2)}_\theta(B^\top f^{(1)}_\theta(z))),
\end{align*}
where $f_\theta^{(1)}:\mathbb{R}^l\to\mathbb{R}^l$ is $L_1$-inverse Lipschitz and $f_\theta^{(2)}:\mathbb{R}^t\to\mathbb{R}^t$ is $L_2$-inverse Lipschitz, both generated by Brenier maps. $B$ is an $l\times t$ diagonal matrix with all diagonal elements with value 1. 
\end{definition}
When $l=t$, we replace $f^{(2)}_\theta(B^\top f^{(1)}_\theta(z))$ with a single $L$-inverse Lipschitz function generated by a Brenier map.
\par We can also establish variants of IL-LIDVAE for mixture models (IL-LIDMVAE) and sequential models (IL-LIDSVAE) as~\citet{WBC2021} did. 
\begin{definition}[IL-LIDMVAE]\label{ilm}
We define IL-LIDMVAE with inverse Lipschitz constant $(L_1,L_2)$ as the following generative model:
\begin{align*}
   y&\sim\mathrm{Categorical}(1/c), z\sim \mathrm{EF}(z\mid B_1^\top y),\\
   x&\sim \mathrm{EF}(x\mid f^{(2)}_\theta(B_2^\top f^{(1)}_\theta(z))),
\end{align*}
where $f_\theta^{(1)}:\mathbb{R}^l\to\mathbb{R}^l$ is $L_1$-inverse Lipschitz and $f_\theta^{(2)}:\mathbb{R}^t\to\mathbb{R}^t$ is $L_2$-inverse Lipschitz, both generated by Brenier maps. $y$ is a one-hot vector that indicates the class. $B_1$ and $B_2$ are respectively a $c\times l$ and $l\times t$ diagonal matrix with all diagonal elements with value 1. 
\end{definition}
\begin{definition}[IL-LIDSVAE]\label{ils}
We define IL-LIDSVAE with inverse Lipschitz constant $(L_1,L_2)$ as the following generative model
\begin{align*}
    z_i\sim p(z),\
    x_i\sim \mathrm{EF}(x\mid f^{(2)}_\theta(B^\top f^{(1)}_\theta(z_i,h_\theta(x_{1:i-1})))),
\end{align*}
where $f_\theta^{(1)}:\mathbb{R}^l\to\mathbb{R}^l$ is $L_1$-inverse Lipschitz and $f_\theta^{(2)}:\mathbb{R}^t\to\mathbb{R}^t$ is $L_2$-inverse Lipschitz, both generated by Brenier maps. $B$ is an $l\times m$ diagonal matrix with all diagonal elements with value 1. 
\end{definition}
\subsection{Discussion}
In practice, the main limitation of this method resides in its computational cost and scalability due to the use of Brenier maps. Since our model is only an extension of LIDVAE \citet{WBC2021} and does not add additional calculation, as mentioned in their paper, fitting LIDVAE or IL-LIDVAE requires a computational complexity of $O(kp^2)$ (that of the classical VAE is $O(kp)$), where $k$ is the number of iterations, and $p$ the number of parameters. The training time is thus longer than the vanilla VAE. While this is the price to pay to assure the identifiability of latent variables and thus to avoid posterior collapse, this algorithm may consequently become difficult to scale for more challenging tasks. Nevertheless, the essential aspect of our model is the inverse Lipschitzness, and the suggested implementation is only the most exact realization of this property. As an approximation, we could consider to replace the Brenier maps and ICNN with a general deep neural network subjected to regularization or constraints, such as $\mathbb E[\|\nabla \log p(x|z)\|]>L$ or $\mathbb E[ \|\log p(x|z)-\log p(x|z’)\|/\|z-z’\|]>L$, that encourages this crucial inverse Lipschitzness. This is outside the scope of this work but we believe it is an interesting avenue for future work.

%% file: sec5.tex
In this section, we will illustrate our theoretical result with several numerical examples. Our goal is to empirically verify (i) whether IL-LIDVAE can indeed control the discrepancy between the posterior and prior with the inverse Lipschitz constant and (ii) if our model can achieve better performance than the vanilla VAE. In order to answer these questions, we will first use toy data and then switch to high-dimensional text and image data. 
\subsection{Toy Data}
\begin{figure}[t]
\begin{center}
\includegraphics[width=40 mm]{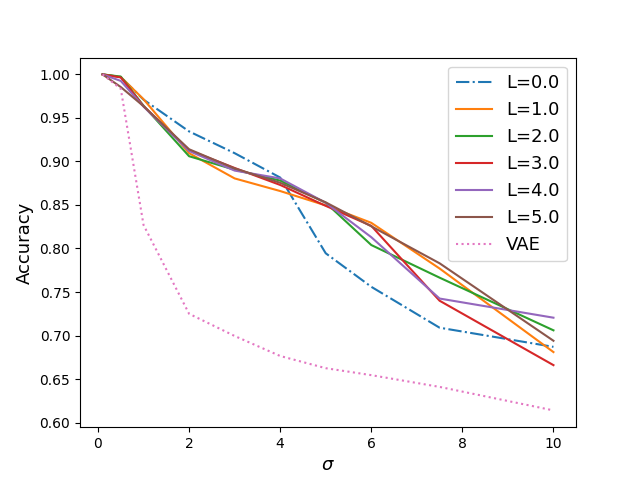}
\includegraphics[width=40 mm]{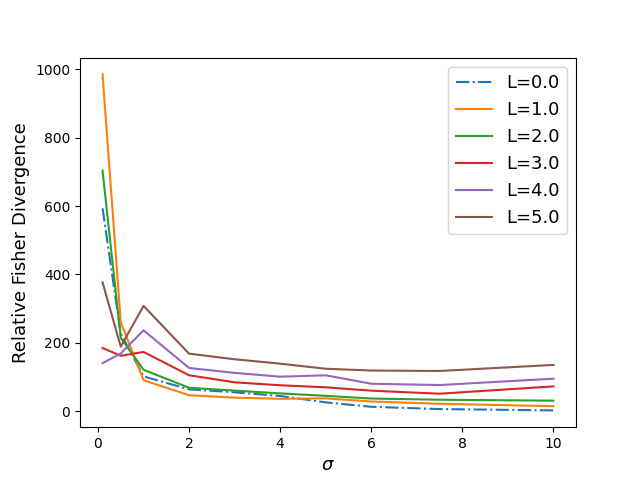}
\caption{Accuracy of the learned posterior (left) and Relative Fisher divergence between the posterior and prior (right) for different standard deviations $\sigma$ and inverse Lipschitz constants $L$. $L=0$ is also the LIDVAE~\citep{WBC2021}. Posterior collapse is happening for LIDVAE but can be controlled with IL-LIDVAE. ``VAE'' in the legend refers to the GMVAE.}
\label{fig:toy1}
\end{center}
%\vskip -0.2in
\end{figure}
\begin{figure}
%\vskip 0.2in
\centering
\includegraphics[width=40 mm]{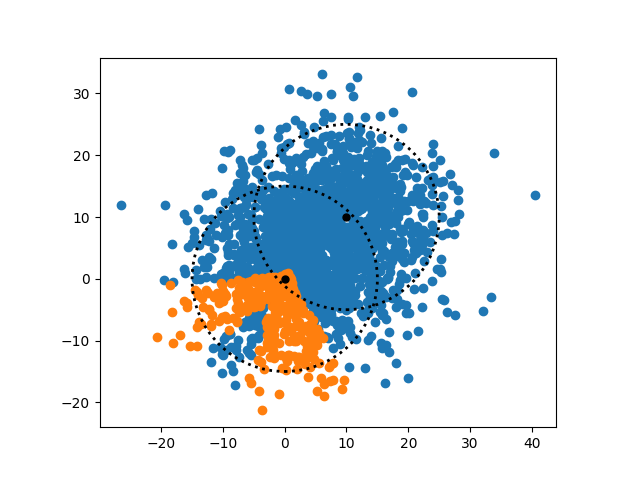}
\includegraphics[width=40 mm]{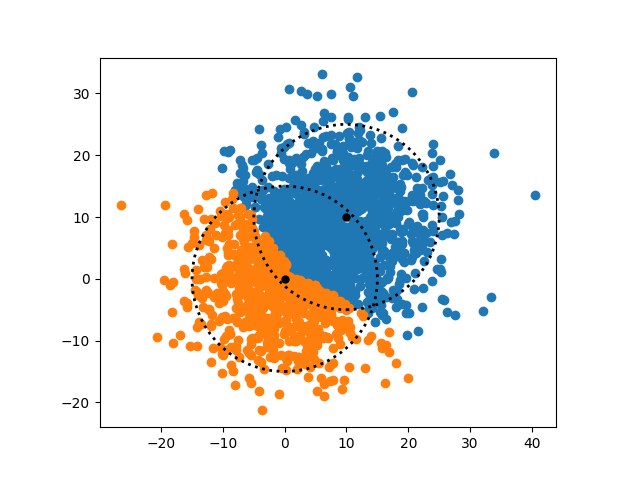}
\caption{Posterior of GMVAE (left) and IL-LIDMVAE with $L_1=L_2=5.0$ (right) for the toy data with $\sigma=7.5$. Black points are the means of $N((0,0)^\top,\sigma^2I_2)$ and $N((10,10)^\top,\sigma^2 I_2)$, and dashed circles delimit the $2\sigma$ area of each distributions. IL-LIDMVAE performs better. See Figure \ref{fig:aptoy} for more data.}
\label{fig:toy2}
%\vskip -0.2in
\end{figure}
We generated 10,000 samples from $\mathcal{N}((0,0)^\top,\sigma^2 I_2)$ and from $\mathcal{N}((10,10)^\top,\sigma^2I_2)$ with different values of $\sigma$, where $I_2$ is the $2\times 2$ identity matrix. With these 20,000 data points, we trained our model IL-LIDMVAE (Definition~\ref{ilm}) defined for Gaussian mixtures with $c=2$. The dimension of the latent variables $l$ was set to 2, and the decoder was also Gaussian. For this case, we only need to control one inverse Lipschitz constant as $l=t$. $L=0$ corresponds to the LIDMVAE of \citet{WBC2021}. Intuitively, for large $\sigma$, the two classes will overlap each other and become similar to a single Gaussian distribution. This is expected to drive the model to represent the whole data with only one Gaussian distribution even though we set to learn two classes of data. IL-LIDMVAE should be able to avoid this with adequate values of $L$.
\par Results are shown in Figure~\ref{fig:toy1}. We used the Gaussian mixture version for VAE, called GMVAE~\citep{DM2016}. For these complicated settings of high variance, the posterior of LIDVAE ($L=0.0$) is collapsing on the prior with a relative Fisher divergence of almost 0, and the accuracy becomes lower than other graphs for $\sigma \geq 5$.  This supports our claim that LIDVAE can only avoid the exact posterior collapse, and that we need a method that can flexibly avoid any degree of $\epsilon$-posterior collapse. Note, the normal VAE does not perform well at all. On the contrary, IL-LIDVAE can adapt to any degree of posterior collapse as higher values of $L$ achieve higher divergence. This leads in general to better accuracy as it is illustrated in Figure \ref{fig:toy2}. 
\subsection{Training on Images and Text}
We used two image data sets, namely Fashion-MNIST~\citep{XR2017} and Omniglot~\citep{LS2015}, and one text data set, namely the synthetic text data provided by~\citet{WBC2021}. IL-LIDMVAE was applied to the image data, and IL-LIDSVAE to the text data. All distributions were set to Gaussian. We compared the performance in terms of the negative log-likelihood and the KL-divergence between the posterior and the prior. We used RealNVPs~\citep{LS2015} for images and LSTMs~\citep{HS1997} for text in order to keep the decoder as flexible as possible. Results are shown in Tables~\ref{tab:syn} and~\ref{tab:fash}.\footnote{We did not cite the result of~\citet{WBC2021} in Table~\ref{tab:fash} for Omniglot because they provided values of the negative log-likelihood (around 600) that were too far from ours (around 100) which are of the same order as those reported by other papers such as \citet{TW2018}.} Further details of the experiments, including additional results, can be found in Appendix~\ref{apb} and in the supplementary material. 
\par As expected, the KL divergence increases in all tables as we augment the value of $L$ and can effectively avoid posterior collapse when it happens. In most cases, the negative log-likelihood is also improved compared to other algorithms. As for Fashion-MNIST, we presented in Figure \ref{fig:sfash} some samples generated from the trained IL-LIDMVAE. Fashion-MNIST contains 10 classes. We can notice that LIDVAE (a) can outputs high-quality data but cannot learn the category of bags, which collapses on several other classes. On the contrary, IL-LIDVAE that achieves the best loss with $L_1=L_2=1.5$ can reproduce these 10 distinct classes with varied data in each one. With too high inverse Lipschitz constants, IL-LIDVAE cannot output high-quality data since the constraint is too strong and the model cannot find good parameters at all. 
\par For Fashion-MNIST, the classification accuracy of the model based on the posterior of the mixture components was 0.56 for $L=0$, 0.58 for $L=0.5$, 0.64 for $L=1.5$ and 0.12 for $L=5.0$. Our model is an unsupervised learning method, and the clustering method for Fashion-MNIST using baseline autoencoders achieves an accuracy of 0.54~\citep{AA2020}, and that using the classical VAE 0.12 in our experiment. As we can observe, the appropriate choice of the inverse Lipschitz constant ($L=1.5$) improves the accuracy over these methods too.
\par In short, we can not only control the degree of posterior collapse but also find adequate parameters of inverse Lipschitzness that improve the performance as well.
\begin{figure}
%\vskip 0.2in
\centering
    \begin{subfigure}{0.15\textwidth}
            \includegraphics[width=\textwidth]{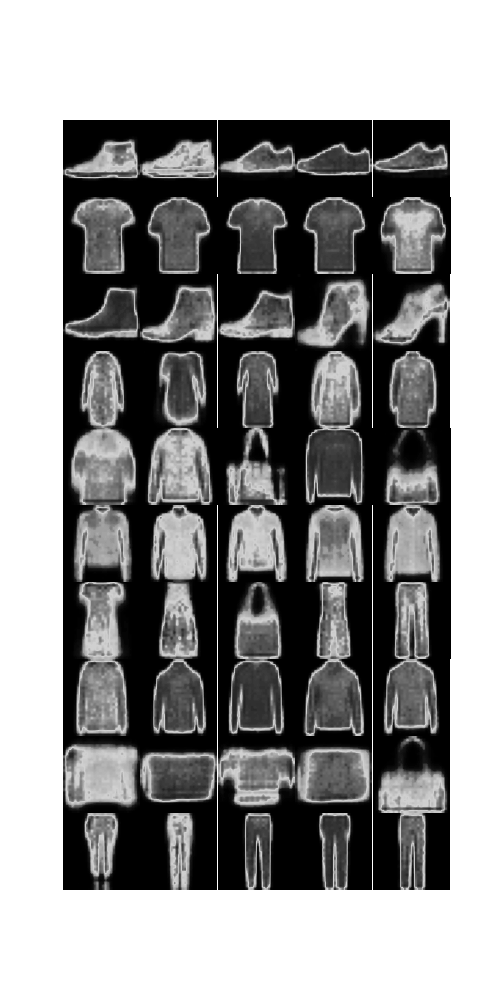}
            \caption{$L_1=L_2=0.0$}
            \label{fig:s00}
    \end{subfigure}
        \begin{subfigure}{0.15\textwidth}
            \includegraphics[width=\textwidth]{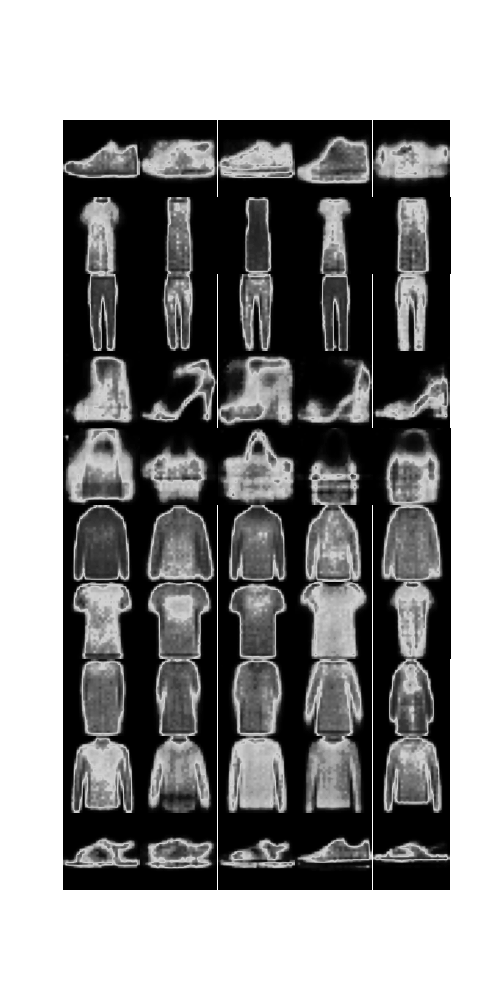}
            \caption{$L_1=L_2=1.5$}
            \label{fig:s15}
    \end{subfigure}
        \begin{subfigure}{0.15\textwidth}
            \includegraphics[width=\textwidth]{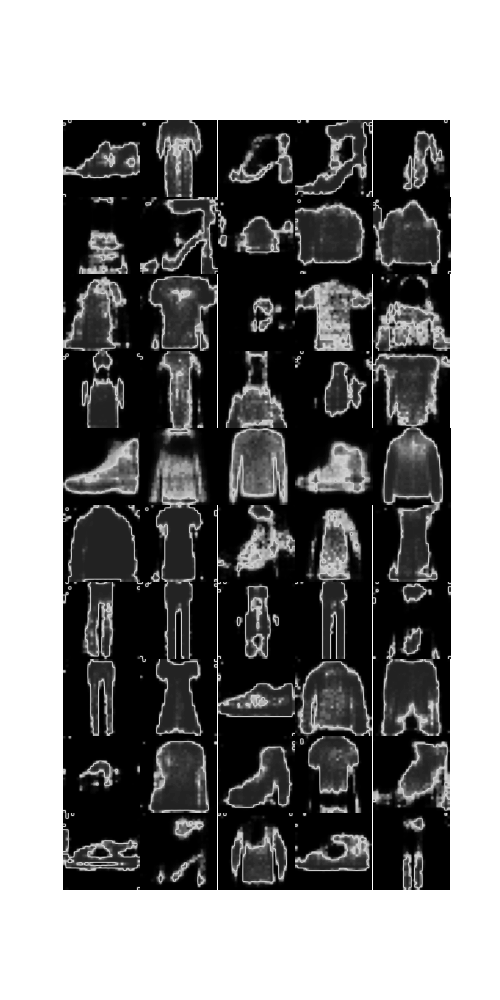}
            \caption{$L_1=L_2=5.0$}
            \label{fig:s50}
    \end{subfigure}
\caption{Samples of Fashion-MNIST data generated with different inverse Lipschitz parameters of IL-LIDMVAE with $c=10$, and all distributions were Gaussian. Each row corresponds to a different category. With $L_1=L_2=1.5$, we obtain the ten true classes with varied images.}
\label{fig:sfash}
%\vskip -0.2in
\end{figure}
\begin{table}[t]
\caption{Results for synthetic text data. $\beta$ was set to 0.2 for $\beta$-VAE. The column entitled $L$ refers to the inverse Lipschitz constant of $f_\theta^{(1)}$ and $f_\theta^{(2)}$. NLL stands for negative log-likelihood. See Table \ref{tab:apsyn} for more data.}
\label{tab:syn}
%\vskip 0.15in
\begin{center}
\begin{small}
\begin{sc}
\begin{tabular}{lccc}
\toprule
Model & $L$ & NLL & KL \\
\midrule
VAE   & - & 42.56 & 0.01  \\
$\beta$-VAE ($\beta=0.2$) & - & 42.34 & 0.08\\
Lagging VAE & - & 45.44 & 2.13\\
LIDSVAE & 0    & 56.67 & 0.24 \\
IL-LIDSVAE &0.5  & \textbf{40.48} & 0.60  \\
IL-LIDSVAE &1.5  & 44.50 & 3.80  \\
IL-LIDSVAE &5.0  & 52.34 & 8.13  \\
\ \ +annealing & - & 39.6 & 0.38\\
\bottomrule
\end{tabular}
\end{sc}
\end{small}
\end{center}
%\vskip -0.1in
\end{table}
\begin{table}[t]
\caption{Results for Fashion-MNIST (Fashion) and Omniglot. The column entitled $L$ refers to the inverse Lipschitz constant of $f_\theta^{(1)}$ and $f_\theta^{(2)}$. NLL stands for negative log-likelihood. $^\ast$ means that the result is cited from \citet{WBC2021}. See Table \ref{tab:apfash} for more data.}
\label{tab:fash}
\vskip 0.15in
\begin{center}
\begin{small}
\begin{sc}
\begin{tabular}{lcccccc}
\toprule
& &\multicolumn{2}{c}{Fashion} & \multicolumn{2}{c}{Omniglot}\\ 
Model & $L$ & NLL & KL & NLL & KL \\
\midrule
VAE$^\ast$ & - & 258.8 & 0.2 & - & - \\
SA-VAE$^\ast$ & - & 252.2 & 0.3 & - & -\\
Lagging VAE$^\ast$ & - & 248.5 & 0.6& -&-\\
$\beta$-VAE$^\ast$ ($\beta=0.2$) & - & 245.3 & 1.2& - & -\\ 
LIDMVAE   & 0 & 237.3 & 9.5 & 135.0 & 15.8\\
IL-LIDMVAE & 0.5  & 240.3 & 10.0  & 129.9 & 20.7\\
IL-LIDMVAE & 1.5  & \textbf{234.4} & 11.0 & \textbf{126.5}& 25.2 \\
IL-LIDMVAE & 5.0  & 243.7 & 14.3 & 128.4 & 26.2 \\
\ \ +annealing & - &  235.6 & 8.1 & 117.7 & 26.0\\ 
\bottomrule
\end{tabular}
\end{sc}
\end{small}
\end{center}
\vskip -0.1in
\end{table}
\subsection{Annealing Method}
An interesting feature of our method is that we can easily consider an annealing scheme that relaxes the inverse Lipschitz constraint when the optimization comes at its boundary (i.e., when the inverse Lipschitz constant of the Brenier map of ICNN becomes close to the set value). That way, we can avoid tuning the inverse Lipschitz constant by hand. Results of this approach are shown in Table \ref{tab:syn} and \ref{tab:fash}, where the negative log-likelihoods are mainly better than results with constant inverse Lipschitz parameters. We only calculated the inverse Lipschitz constant of the first layer $f_\theta^{(1)}$ and estimated it by recycling inputs and outputs acquired during the training, which does not considerably increase the computational complexity. The decrease rate of the inverse Lipschitz constant was set to 0.85. Although we did not evaluate quantitatively, we observed that the final value of the parameter remained around that we found with the best negative log-likelihood without annealing.

%% file: sec6.tex
In conclusion, starting from the recent observation that posterior collapse and latent variable non-identifiability are related to each other, we investigated a method that can guarantee to mitigate any degree of posterior collapse engendered from the formulation of the model itself. This was achieved by introducing an inverse Lipschitz neural network into the decoder that can freely control the degree of latent variable identifiability, and thus that of posterior collapse. Indeed, we theoretically proved that, as this constant increases, the posterior is moved away from the prior in terms of the relative Fisher information divergence in a non-decreasing manner. Based on this theoretical guarantee, we expanded our method to the algorithm IL-LIDVAE and its variants, which are applicable to a broad range of problem settings and can be easily tuned to avoid posterior collapse. We applied them to synthetic and real-world data and showed that they could clearly control the discrepancy between the posterior and the prior in agreement with the theoretical analysis, which was never explicitly realized in any prior work. In most cases, this also had the effect of finding better local minima with lower loss than that of the vanilla VAE.

%% file: apa.tex
In this appendix, we will prove the theorems stated in the main paper.
\subsection{Proof of Theorem~\ref{th1}}
\label{sec:proof_main1}
Let us first state some simple lemmas to prove Theorem~\ref{th1}. While the proofs may seem simple, we will still show most of them in order to keep our paper self-contained.
\begin{lemma}\label{l1}
The following holds between the prior $p(z)$, the posterior $p_\theta(z\mid x)$, the likelihood $p_\theta(x\mid z)$ and the marginal likelihood $p_\theta(x)$:
\[
\|\nabla_z \log p_\theta(z\mid x)-\nabla_z \log p(z)\|=\|\nabla_z \log p_\theta(x\mid z)\|.
\]
\end{lemma}
\begin{proof}
Bayes' theorem implies 
\[
p_\theta(x\mid z)=\frac{p_\theta(x,z)}{p(z)}=\frac{p_\theta(z\mid x)p_\theta(x)}{p(z)}.
\]
Taking logarithms on both sides,
\[
\log p_\theta(x\mid z)=\log p_\theta(z\mid x)+\log p_\theta(x) -\log p(z),
\]
which leads to 
\[
\log p_\theta(x\mid z)-\log p_\theta(x)=\log p_\theta(z\mid x)-\log p(z).
\]
Now, differentiating with respect to $z$ and taking the norm gives the desired equality since $p_\theta(x)$ is independent of $z$.
\end{proof}
\begin{remark}
Note this theorem does not need any specification on the class of distributions.
\end{remark}
The following lemma is a fundamental property of the exponential family.
\begin{lemma}\label{l2}
    The following holds for the exponential family between the sufficient statistic and log-partition function:
    \[
    \mathbb{E}_{x}[T(x)]=\nabla_z A(z).
    \]
\end{lemma}
Furthermore, the generation of inverse-Lipschitz functions by Brenier maps implies the following property.
\begin{lemma}\label{l3}
    If $f:\mathbb{R}^l\to\mathbb{R}^l$ is an $L$-inverse Lipschitz function generated by an $L$-strongly convex real-valued function $B$, then the following holds:
    \[
    \nabla f(x)\succeq LI_l,
    \]
    where $I_l$ is the $l\times l$ identity matrix, and $A\succeq B$ means that $A-B$ is positive semi-definite.
\end{lemma}
\begin{proof}
    Since $f=\nabla B$ and $\nabla f=\nabla^2 B$, by definition of $L$-strong convexity, we immediately get $\nabla f(x)\succeq LI_l$.
\end{proof}
\begin{remark}
    This statement can be proved thanks to the use of Brenier maps in order to create inverse Lipschitz functions.
\end{remark}
Finally, we can prove Theorem~\ref{th1}.
\begin{theoremr*}\label{ath1}
Under model \eqref{model}, Assumption~\ref{as1} and $l=t$, the following holds for all $i$ and $\theta\in\Theta_L$:
\begin{align}\label{aeq25}
    F(p(z)\mid\mid p_\theta(z\mid x_i))\ge& L^2\int\|T(x_i)-\E_{p_\theta(x\mid z)}[T(x)]\|^2 p(z)\d z.
\end{align}
\end{theoremr*}
\begin{proof}
    Since
    \[
    F( p(z)\mid\mid p_\theta(z\mid x))=\int\|\nabla_z \log p_\theta(z\mid x)-\nabla_z \log p(z) \|^2p(z)\d z,
    \]
    Previous lemmas imply
    \begin{align*}
    F( p(z)\mid\mid p_\theta(z\mid x))=&\int\|\nabla_z\log p_\theta(x\mid z) \|^2p(z)\d z\\
    =& \int\|\nabla_z\left(\log h(x)+ f_\theta(z)^\top T(x) -A(f_\theta(z))\right) \|^2p(z)\d z\\
    =& \int\|  \nabla_z f_\theta(z)^\top T(x)-\nabla_z A(f_\theta(z)) \|^2p(z)\d z\\
    =& \int\| \nabla_z f_\theta(z)^\top T(x) -\nabla_z f_\theta(z)^\top \nabla A(f_\theta(z))\|^2p(z)\d z\\
    =& \int\| \nabla_z f_\theta(z)^\top\left(T(x)-\nabla A(f_\theta(z))\right)\|^2p(z)\d z\\
    =& \int\left(T(x)-\nabla A(f_\theta(z))\right)^\top\nabla_z f_\theta(z) \nabla_z f_\theta(z)^\top\left(T(x)-\nabla A(f_\theta(z))\right)p(z)\d z\\
    \ge & L^2\int\|T(x)-\nabla A(f_\theta(z))\|^2p(z)\d z\\
    =& L^2\int\|T(x)-\E_{p_\theta(x\mid z)}[T(x)]\|^2p(z)\d z.
    \end{align*}
    We used Lemma \ref{l1} for the first equality, Lemma \ref{l3} for the inequality, and Lemma \ref{l2} for the last equality. 
\end{proof}
\begin{corr*}\label{acor1}
Under model~\eqref{model}, Assumption~\ref{as1} and $l=t$,
\begin{align}
    F( p(z)\mid\mid p_\theta(z\mid x_i))\ge L^2\inf_{\theta\in\Theta_L}\left\{\int \|T(x_i)-\E_{p_\theta(x\mid z)}[T(x)]\|^2 p(z)\d z\right\}\nonumber.
\end{align}
Therefore, the lower bound is non-decreasing in terms of L. Moreover, if the infinimum of this lower bound is attained by a parameter $\theta\in\Theta_L$ and $p(z)$ has a positive variance, then the lower bound is monotonically increasing in terms of L.
\end{corr*}
\begin{proof}
The infimum immediately follows from equation \eqref{aeq25} which holds for all $\theta\in \Theta_L$. Now, since if $L\ge L'$, then $\Theta_L\supset \Theta_L'$ by definition of inverse Lipschitzness, the infimum is non-decreasing in terms of $L$.
\par Concerning the second half of the statement, let us suppose that the infinimum can be attained by a parameter $\theta\in\Theta_L$ and that $p(z)$ has a positive variance. It is enough to show that the infinimum is not 0, or in other words, that the lower bound is not vacuous. However if 
\[
\int \|T(x_i)-\E_{p_\theta(x\mid z)}[T(x)]\|^2 p(z)\d z=0,
\]
this means the random variable $\E_{p_\theta(x\mid z)}[T(x)]$ satisfies $\E_{p_\theta(x\mid z)}[T(x)]=T(x_i)$ for all $z$, which implies $\E_{p_\theta(x\mid z)}[T(x)]$ is constant. This is only possible if $\E_{p_\theta(x\mid z)}[T(x)]$ does not depend on $z$, which means $p_\theta(x\mid z)$ is independent of $z$. This contradicts our definition of $f_\theta$ which is injective. Therefore, the infimum is positive, and the lower bound is, in consequence, increasing in terms of $L$.
\end{proof}
\subsection{Proof of Theorem \ref{th2}}
\label{sec:proof_main2}
\begin{theoremr2*}\label{ath2}
Under model~\eqref{model}, Assumption~\ref{as1} and $l=t$, the following holds for all $\theta\in\Theta_L$:
\begin{align*}
    \bar{F}_\theta(\bm{x})\vcentcolon=\int\left\|\frac{1}{n
}\sum_{i=1}^n\nabla_z \log p_\theta(z\mid x_i)-\nabla_z \log p(z)\right\|^2p(z)\d z\ge L^2\int\left\|\frac{1}{n}\sum_{i=1}^nT(x_i)-\E_{p_\theta(x\mid z)}[T(x)]\right\|^2p(z)\d z.
\end{align*}
\end{theoremr2*}
\begin{proof}
    It suffices to note that 
\begin{align*}
        \frac{1}{n}\sum_{i=1}^n\nabla_z\log p_\theta(z\mid x_i)=&\frac{1}{n}\sum_{i=1}^n\nabla_z\left(\log h(x_i)+ f_\theta(z)^\top T(x_i) -A(f_\theta(z))\right)\\
        =&\nabla_z\left(f_\theta(z)^\top\frac{1}{n}\sum_{i=1}^n T(x_i) -A(f_\theta(z))\right).
        \qedhere
\end{align*}
\end{proof}
\subsection{Proof of Theorem~\ref{th3}}\label{apa4}
This subsection considers IL-LIDVAE in the general case where $t >l$, i.e., the latent dimension $l$ is smaller than that of the sufficient statistic of $x$ (Definition \ref{def:IL-LIDVAE}), and discusses the lower bound of the relative Fisher information divergence between the prior and posterior. 

Let us first remind Assumption \ref{as1} and the problem setting when $t> l$. $\mathrm{EF}(x\mid \xi)$ is an exponential family defined by 
\begin{equation*}
\mathrm{EF}_T(x\mid \xi)=\exp\{ T(x)^\top \xi - A(\xi)\}h(x),
\end{equation*}
where $T(x)=(T_1(x),\ldots,T_t(x))$ is a set of sufficient statistics, $\xi$ a natural parameter, and $h(x)$ a base probability density. We assume $h(x)>0$ and that the natural parameter $\xi$ is defined on an open set in $\mathbb{R}^l$. The IL-LIDVAE model is defined by $p_\theta(x\mid z)=\mathrm{EF}_T(x\mid f_\theta (z))$, where
\[
f_\theta(z) = f_\theta^{(2)}(B^\top f_\theta^{(1)}(z))
\]
which is given by inverse Lipschitz functions $f_\theta^{(1)}$ and $f_\theta^{(2)}$ with constant $L_1$ and $L_2$, respectively. The model can be regarded as a curved exponential family parameterized with $f_\theta(z)$.

We make a natural assumption on the exponential family.
\begin{assumption}\label{assump:diffeo}
The mapping from the natural parameter to the expectation parameter 
\begin{equation*}
\xi \mapsto \nabla_\xi A(\xi)=\E_{\mathrm{EF}_T(x\mid \xi)}[T(x)]
\end{equation*}
is a diffeomorphism.  
\end{assumption}
In light of the relation
\[
\nabla_\xi \log \mathrm{EF}_T(x\mid \xi) = T(x) - \nabla_\xi A(\xi) = T(x) -\E_{\mathrm{EF}_T(x\mid \xi)}[T(x)],
\]
the above assumption means that the natural parameter $\xi$ effectively changes the density function for any direction of the parameter space. This holds for many popular exponential families, such as Gaussian distributions. 
\par Under this problem setting, we can prove Theorem~\ref{th3}, which is restated more explicitly.
\begin{theorem}[Theorem~\ref{th3} restated]
Under Assumptions \ref{as1} and \ref{assump:diffeo}, and $l=t$,
\[
F(p(z)\mid\mid p_\theta(z\mid x))\geq L_1^2\inf_{\theta\in\Theta_{L_1,L_2}}\left\{\int \left\Vert \bigl(T(x)-\E_{p_\theta(x\mid z)}[T(x)]\bigr)^\top  \nabla f_\theta^{(2)}\bigl(B^\top f_\theta^{(1)}(z)\bigr) B^\top \right\Vert^2 p(z) dz\right\},
\]
where the right-hand side of the inequality is increasing in terms of $L_1$.
\end{theorem}
\begin{proof}
As in the proof of Theorem \ref{th1}, the relative Fisher information divergence $F(p(z)\mid\mid p_\theta(z\mid x))$ is lower bounded by 
\begin{align*}
F(p(z)\mid\mid p_\theta(z\mid x)) & = \int \Vert \nabla_z \log p_\theta(x\mid z) \Vert^2 p(z) dz \\
& = \int \Vert \nabla_\xi \log \mathrm{EF}_T(x\mid \xi)|_{\xi=f_\theta(z)} \nabla_z f_\theta(z) \Vert^2 p(z) dz \\
& = \int \left\Vert \bigl(T(x)-\E_{p_\theta(x\mid z)}[T(x)]\bigr)^\top  \nabla f_\theta^{(2)}\bigl(B^\top f_\theta^{(1)}(z)\bigr) B^\top \nabla f_\theta^{(1)}(z)) \right\Vert^2 p(z) dz \\
& \geq L_1^2 \int \left\Vert \bigl(T(x)-E_{p_\theta(x\mid z)}[T(x)]\bigr)^\top  \nabla f_\theta^{(2)}\bigl(B^\top f_\theta^{(1)}(z)\bigr) B^\top \right\Vert^2 p(z) dz.
\end{align*}

To guarantee the lower boundedness of the divergence with the control by the inverse-Lipschitz constant $L_1$, the integral in the last line should be positive. We consider this positiveness under the assumption that the density of the prior $p(z)$ is everywhere positive and continuous. In this case, the integral is positive if and only if the latent space has an open set on which  
\[
\bigl(T(x)-\E_{p_\theta(x\mid z)}[T(x)]\bigr)^\top  \nabla f_\theta^{(2)}\bigl(B^\top f_\theta^{(1)}(z)\bigr) B^\top  \neq 0.
\] 
Because $\nabla f_\theta^{(1)}$ is invertible, 
%by the inverse-Lipschitzness condition of $f^{(1)}$, 
this is equivalent to 
\[
\nabla_z \log p_\theta(x\mid z) \neq 0
\]
on that open set. This holds under the assumption because the parameter $\xi=f_\theta(z)$ moves $t$-dimensional directions as $z$ changes, and it in turns changes $\log p_\theta(x\mid z)$ by Assumption \ref{assump:diffeo}. This implies the desired result.
\end{proof}
\subsection{Proof of Proposition \ref{prop1}}
\label{sec:proof_prop}
\begin{propositionr*}
Suppose that the lower bound of Fisher divergence $F(p(x) \mid\mid q(x)) \geq \varepsilon$ holds for any small perturbations of $p$ and $q$ to some extent. More precisely, let $p_t$ (or $q_t$) denote the convolution between $p$ ($q$, resp.) and $N(0,t)$. Assume that there is $\delta > 0$ such that $F(p_t \mid\mid q_t) \ge \epsilon$ for any $t\in[0,\delta]$. Then, the bound $D(p \mid\mid q) \ge \frac{1}{2}\delta \epsilon$ holds.
\end{propositionr*}
\begin{proof}
This is a straightforward consequence of the well-known de Bruijn's identity \citep{L2009} about the relation between KL and Fisher divergences: $\frac{\d}{\d t} D(p_t \mid\mid  q_t ) = -\frac{1}{2} F(p_t \mid\mid  q_t)$. By integrating both sides, we obtain $D(p_\delta \mid\mid  q_\delta) - D(p\mid\mid q) = -\frac{1}{2} \int_0^\delta F(p_t \mid\mid  q_t)\d t$, and this implies $D(p\mid\mid q) \geq \frac{1}{2} \int_0^\delta F(p_t \mid\mid  q_t)dt \ge \frac{1}{2}\delta\epsilon$, which concludes the assertion.
\end{proof}

%% file: apb.tex
In this appendix, we provide further details on experiments conducted in this paper.
\subsection{Details of Data and Experiments}
\begin{table}[t]
\caption{Details of Experiments}
\label{tab:apexp}
\vskip 0.15in
\begin{center}
\begin{small}
\begin{sc}
\begin{tabular}{lcccc}
\toprule
Data & Model & Decoder & Latent Dimension & Size of hidden layers\\
\midrule
Toy   &IL-LIDMVAE ($c=2$) & Gaussian & 2  & 10 \\
Fashion-MNIST &IL-LIDMVAE ($c=10$) & RealNVP (2 layers) & 64 & 512 \\
Omniglot & IL-LIDMVAE ($c=50$)& RealNVP (2 layers)  & 32 & 200\\
Synthetic & IL-LIDSVAE    & LSTM (2 layers) & 1024 & 1024\\
\bottomrule
\end{tabular}
\end{sc}
\end{small}
\end{center}
\vskip -0.1in
\end{table}
\paragraph{Image}
For image data sets, we used the IL-LIDMVAE (Definition~\ref{ilm}) with the exponential family defined as Gaussian distribution. For all other methods, we used their Gaussian mixture variant as GMVAE for VAE. The number of categories $c$ was set to the true number of classes of the data. Fashion-MNIST contains 10 classes (T-shirt/top, trouser, pullover, dress, coat, sandal, shirt, sneaker, bag and ankle boot), and Omniglot 50. We parameterized the prior distribution as well. The two Lipschitz constants $L_1$ and $L_2$ were equal throughout the experiments. Further details can be found in Table \ref{tab:apexp}.
\paragraph{Text}
For the text data, we used IL-LIDSVAE (Definition \ref{ils}). The synthetic data set was generated from a two-layer sequential VAE with five-dimensional latent variables by~\citet{WBC2021}. The two Lipschitz constants $L_1$ and $L_2$ were equal throughout the experiments. Further details can be found in Table \ref{tab:apexp}.
\subsection{Further Results of Experiments}
\subsubsection{Toy data}
We show in Figure \ref{fig:aptoy} the posterior of some trained models.
\begin{figure}
\vskip 0.2in
\centering
    \begin{subfigure}{0.3\textwidth}
            \includegraphics[width=\textwidth]{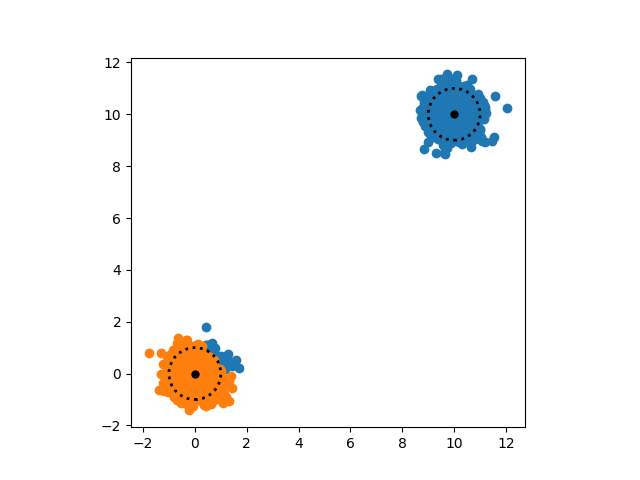}
            \caption{VAE}
            \label{fig:toy05vae}
    \end{subfigure}
    \begin{subfigure}{0.3\textwidth}
            \includegraphics[width=\textwidth]{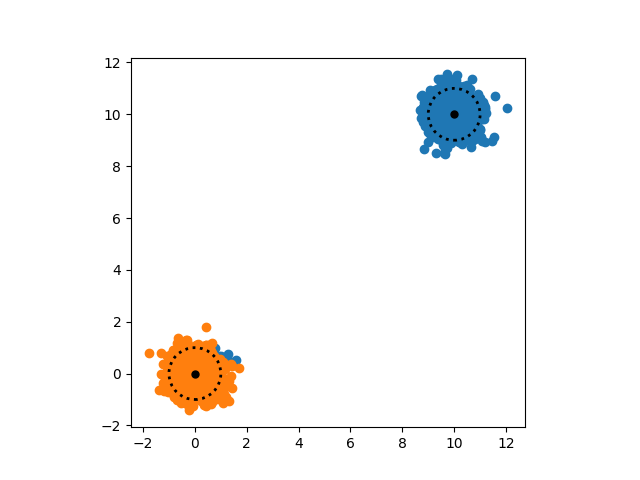}
            \caption{LIDVAE ($L_1=L_2=0$)}
            \label{fig:toy05idvae}
    \end{subfigure}
        \begin{subfigure}{0.3\textwidth}
            \includegraphics[width=\textwidth]{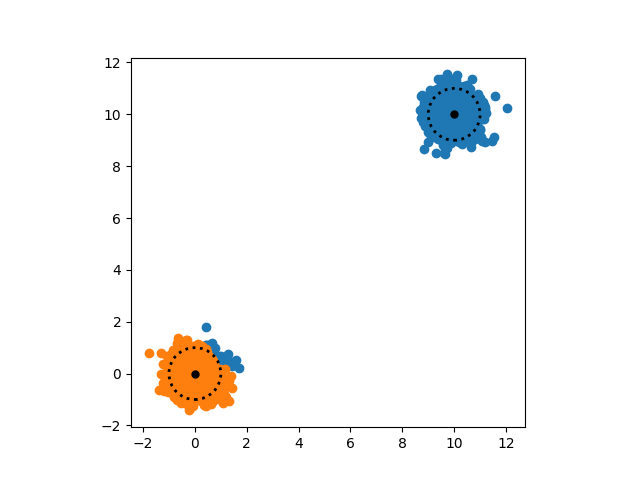}
            \caption{IL-LIDVAE ($L_1=L_2=5.0$)}
            \label{fig:toy05ilvae}
    \end{subfigure}\\
        \begin{subfigure}{0.3\textwidth}
            \includegraphics[width=\textwidth]{out255_toy27.5_vae0.0001_1.5_1.5_r199toy_pointswc.png}
            \caption{VAE}
            \label{fig:toy7vae}
    \end{subfigure}
        \begin{subfigure}{0.3\textwidth}
            \includegraphics[width=\textwidth]{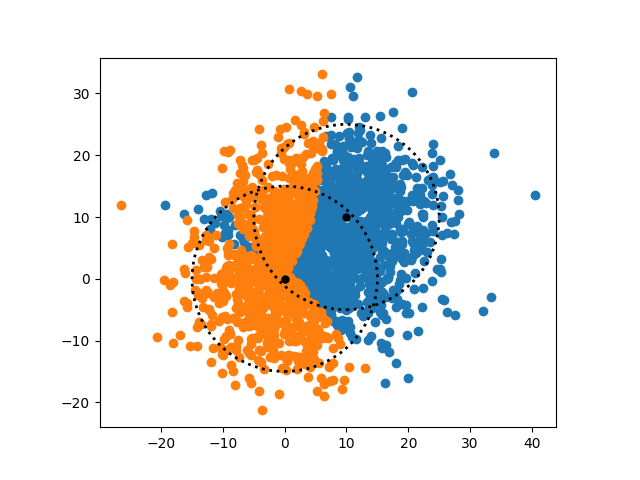}
            \caption{LIDVAE ($L_1=L_2=0$)}
            \label{fig:toy7idvae}
    \end{subfigure}
        \begin{subfigure}{0.3\textwidth}
            \includegraphics[width=\textwidth]{out255_toy27.5_idvae0.0001_5.0_5.0_r199toy_pointswc.png}
            \caption{IL-LIDVAE ($L_1=L_2=5.0$)}
            \label{fig:toy7ilvae}
    \end{subfigure}
\caption{Posterior of VAE (left), LIDVAE (middle) and IL-LIDVAE (right) for the toy data with different standard deviations. $\sigma=0.5$ (top) and $\sigma=7.5$ (bottom). The black points are the means of $N((0,0)^\top,\sigma^2I_2)$ and $N((10,10)^\top,\sigma^2 I_2)$, and the dashed circles delimit the $2\sigma$ area of each distributions.}
\label{fig:aptoy}
\vskip -0.2in
\end{figure}
All models perform well for moderate values of $\sigma$, but only IL-LIDVAE can adapt to the more extreme cases. Data were generated from $N((0,0)^\top,\sigma^2I_2)$ and $N((10,10)^\top,\sigma^2 I_2)$. Therefore, data should be separated along the mediator of the segment connecting $(0,0)^\top$ and $(10,10)^\top$. IL-LIDVAE with $L_1=L_2=5.0$ is the closest to this situation.
\subsubsection{Images}
We show in Table \ref{tab:apfash} additional data with different Lipschitz constants. The mutual information (MI) between the data and the latent variables \citep{HJ2016} and the percentage of active units (AU) \citep{BGS2015}, two alternatives measure of posterior collapse, are also presented in the table. For the percentage of active units, the threshold was set to 0.01.
\begin{table}[t]
\caption{Results for Fashion-MNIST (Fashion) and Omniglot. The column entitled $L$ refers to the inverse Lipschitz constant of $f_\theta^{(1)}$ and $f_\theta^{(2)}$. NLL stands for negative log-likelihood, MI for mutual information and AU for active units. $^\ast$ means that the result is cited from \citet{WBC2021}.}
\label{tab:apfash}
\vskip 0.15in
\begin{center}
\begin{small}
\begin{sc}
\begin{tabular}{lcccccccccc}
\toprule
& &\multicolumn{4}{c}{Fashion} & \multicolumn{4}{c}{Omniglot}\\ 
Model & $L$ &  NLL & KL & MI & AU & NLL & KL &MI & AU\\
\midrule
VAE$^\ast$ & - & 258.8 & 0.2 & 0.9 & 0.1 & - & - & - & -\\
SA-VAE$^\ast$ & - & 252.2 & 0.3 & 1.3 & 0.2 & - & -& - & -\\
Lagging VAE$^\ast$ & - & 248.5 & 0.6& 1.6 & 0.4&-&-& - & -\\
$\beta$-VAE$^\ast$ ($\beta=0.2$) & - & 245.3 & 1.2& 2.4 & 0.6 & - & -& - & -\\ 
LIDVAE   & 0 & 237.3 & 9.5 & 8,7 & 1.0 & 135.0 & 15.8 & 15.2 & 1.0\\
IL-LIDVAE & 0.5  & 240.3 & 10.0 & 9.4 & 1.0  & 129.9 & 20.7 & 20.2 & 1.0\\
IL-LIDVAE & 1.5  & \textbf{234.4} & 11.0 & 10.7 & 1.0 & 126.5 & 25.2 & 24.8 & 1.0 \\
IL-LIDVAE & 2.5  & 236.0 & 12.6 & 12.5& 1.0 & \textbf{126.2} & 25.9 & 25.3 & 1.0\\
IL-LIDVAE & 3.5  & 239.7 & 13.0& 12.9 &0.5 & 126.5 & 26.3 & 26.2 & 1.0\\
IL-LIDVAE & 4.5  & 241.7 & 13.2&13.2 &0.4 & 127.7 & 26.3 & 26.2 & 1.0\\
IL-LIDVAE & 5.0  & 243.7 & 14.3& 14.2 & 0.3 & 128.4 & 26.2 & 26.1 & 1.0\\
\ \ +annealing & - &  235.6 & 8.1 &7.8  & 1.0 & 117.7 & 26.0 &25.8 &1.0\\ 
\bottomrule
\end{tabular}
\end{sc}
\end{small}
\end{center}
\vskip -0.1in
\end{table}
We also show some reconstruction and sampling conducted with the trained models of IL-LIDVAE in Figures \ref{fig:aprecfash} and \ref{fig:apsamfash}, respectively.
\begin{figure}
\vskip 0.2in
\centering
    \begin{subfigure}{0.2\textwidth}
            \includegraphics[width=\textwidth]{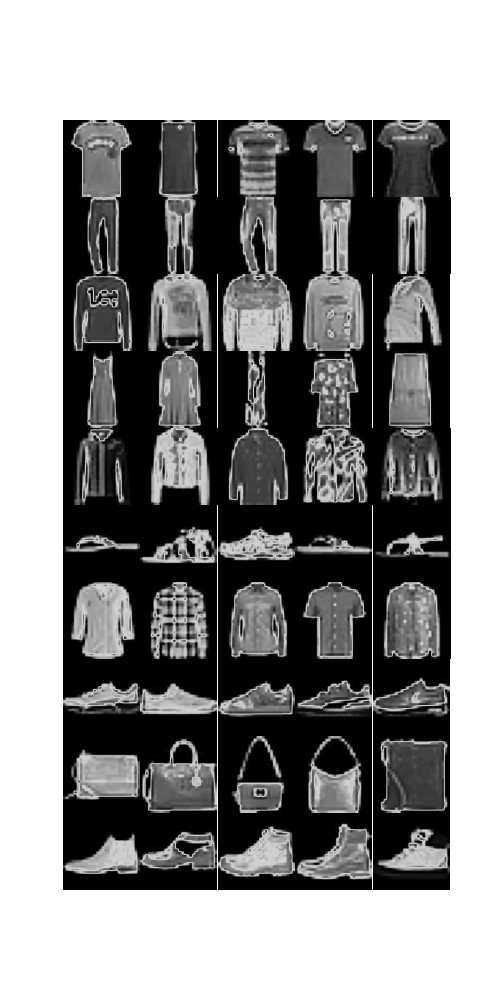}
            \caption{True Data}
            \label{fig:true}
    \end{subfigure}
    \begin{subfigure}{0.2\textwidth}
            \includegraphics[width=\textwidth]{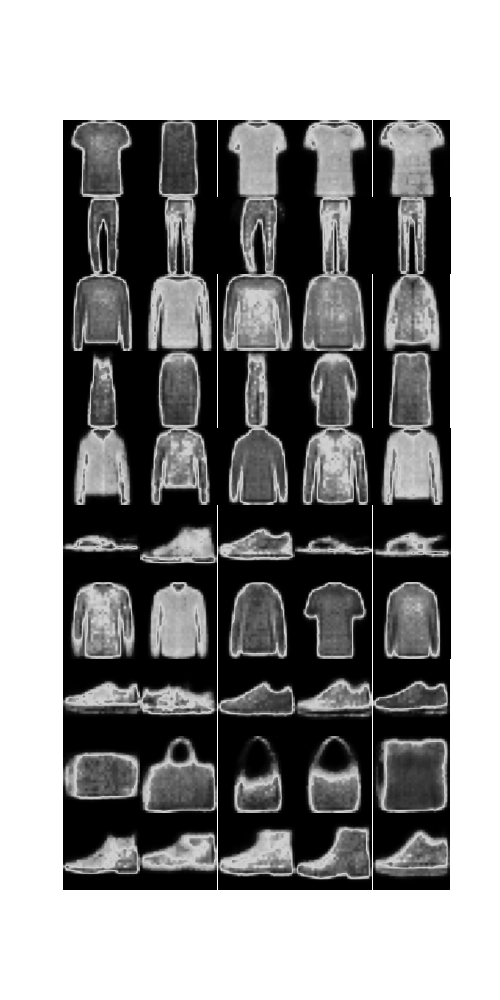}
            \caption{$L_1=L_2=0.0$}
            \label{fig:rec00}
    \end{subfigure}
    \begin{subfigure}{0.2\textwidth}
            \includegraphics[width=\textwidth]{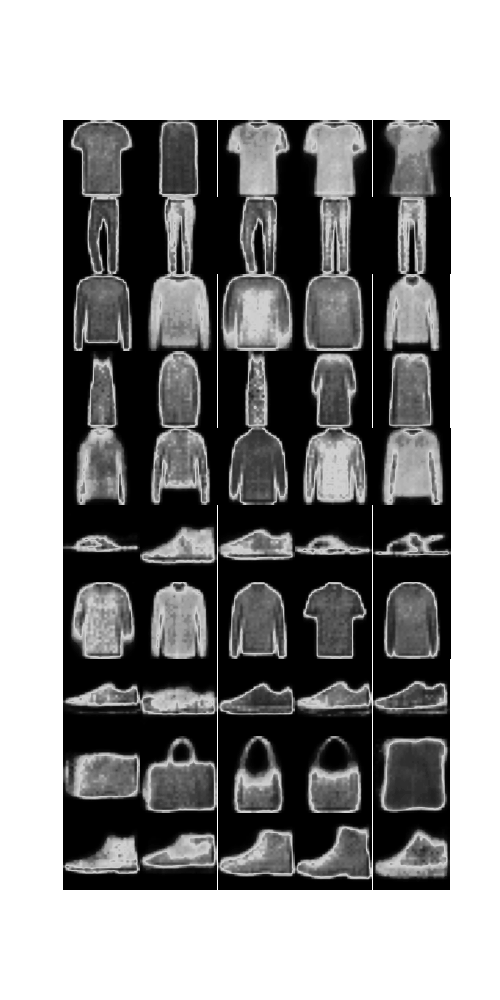}
            \caption{$L_1=L_2=0.5$}
            \label{fig:rec05}
    \end{subfigure}
        \begin{subfigure}{0.2\textwidth}
            \includegraphics[width=\textwidth]{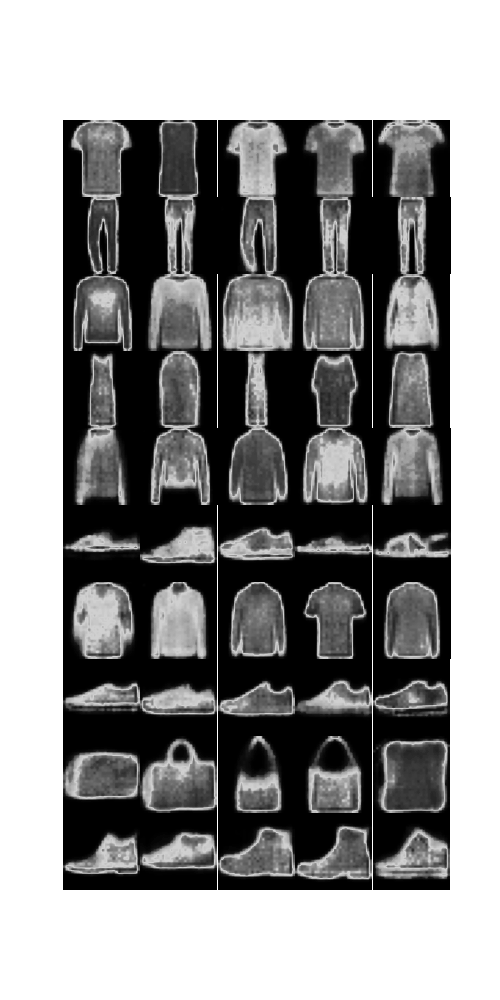}
            \caption{$L_1=L_2=1.5$}
            \label{fig:rec15}
    \end{subfigure}
        \begin{subfigure}{0.2\textwidth}
            \includegraphics[width=\textwidth]{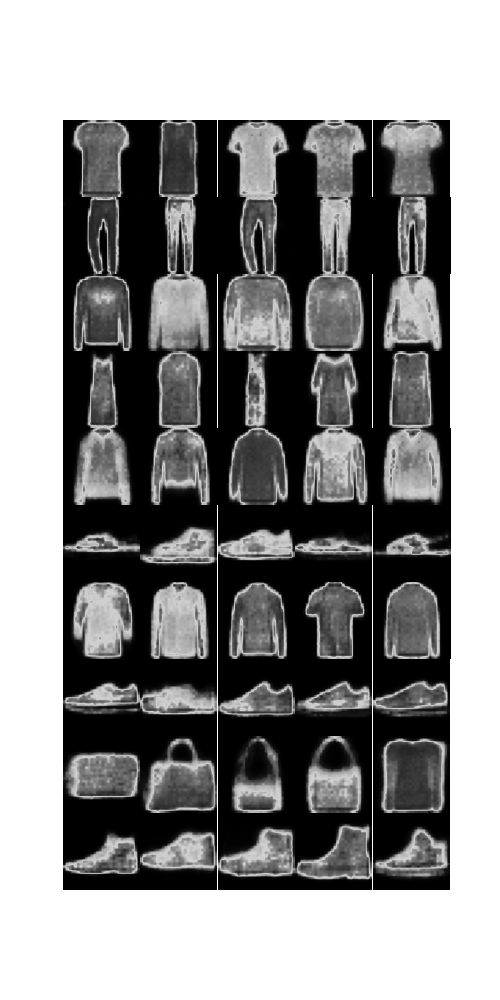}
            \caption{$L_1=L_2=2.5$}
            \label{fig:rec25}
    \end{subfigure}
        \begin{subfigure}{0.2\textwidth}
            \includegraphics[width=\textwidth]{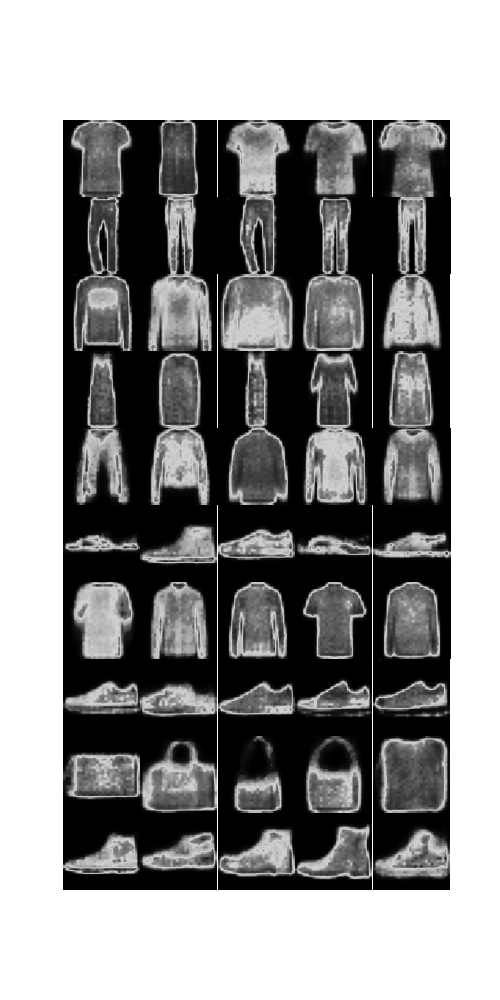}
            \caption{$L_1=L_2=3.5$}
            \label{fig:rec35}
    \end{subfigure}
        \begin{subfigure}{0.2\textwidth}
            \includegraphics[width=\textwidth]{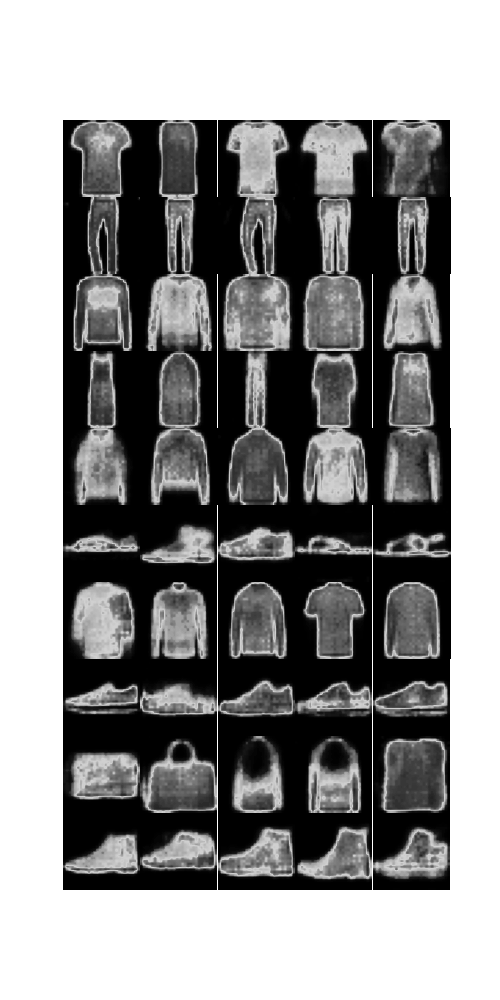}
            \caption{$L_1=L_2=4.5$}
            \label{fig:rec45}
    \end{subfigure}
        \begin{subfigure}{0.2\textwidth}
            \includegraphics[width=\textwidth]{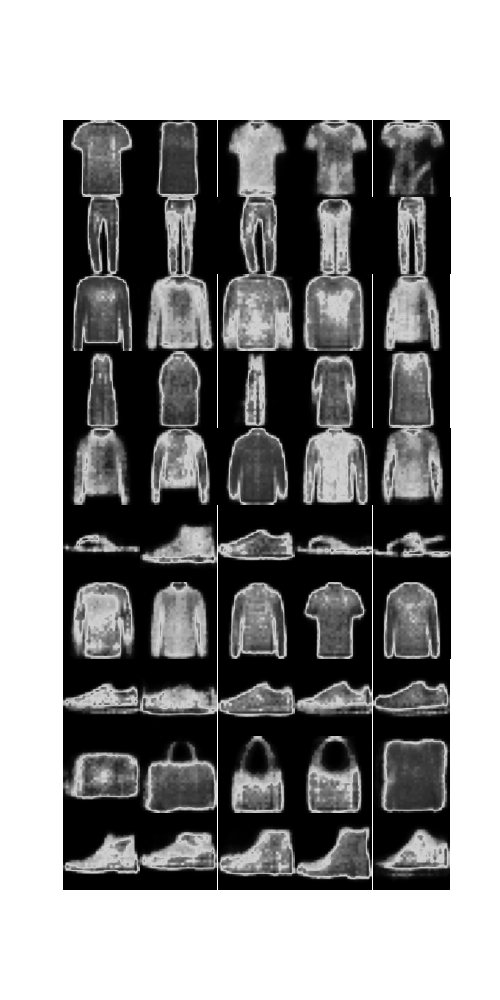}
            \caption{$L_1=L_2=5.0$}
            \label{fig:rec50}
    \end{subfigure}
\caption{Reconstruction of randomly chosen data of Fashion-MNIST for different inverse Lipschitz parameters of IL-LIDMVAE. The number of classes was set to $10$, and all distributions were Gaussian. Each row corresponds to a different category. }
\label{fig:aprecfash}
\vskip -0.2in
\end{figure}
\begin{figure}
\vskip 0.2in
\centering
    \begin{subfigure}{0.2\textwidth}
            \includegraphics[width=\textwidth]{out_fashionmnist_idvae0.0005_0.0_0.0_r10sample_summary.png}
            \caption{$L_1=L_2=0.0$}
            \label{fig:sam00}
    \end{subfigure}
    \begin{subfigure}{0.2\textwidth}
            \includegraphics[width=\textwidth]{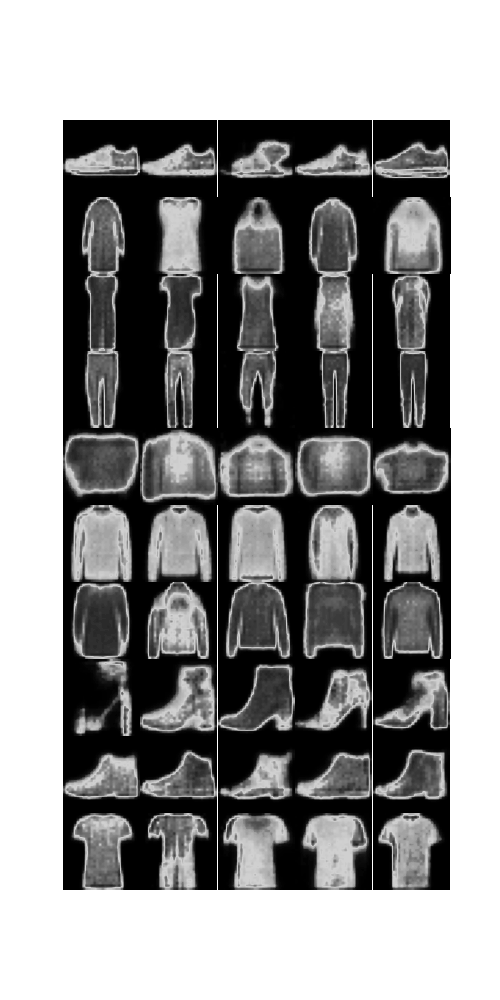}
            \caption{$L_1=L_2=0.5$}
            \label{fig:sam05}
    \end{subfigure}
        \begin{subfigure}{0.2\textwidth}
            \includegraphics[width=\textwidth]{out_fashionmnist_idvae0.0005_1.5_1.5_r10sample_summary.png}
            \caption{$L_1=L_2=1.5$}
            \label{fig:sam15}
    \end{subfigure}
        \begin{subfigure}{0.2\textwidth}
            \includegraphics[width=\textwidth]{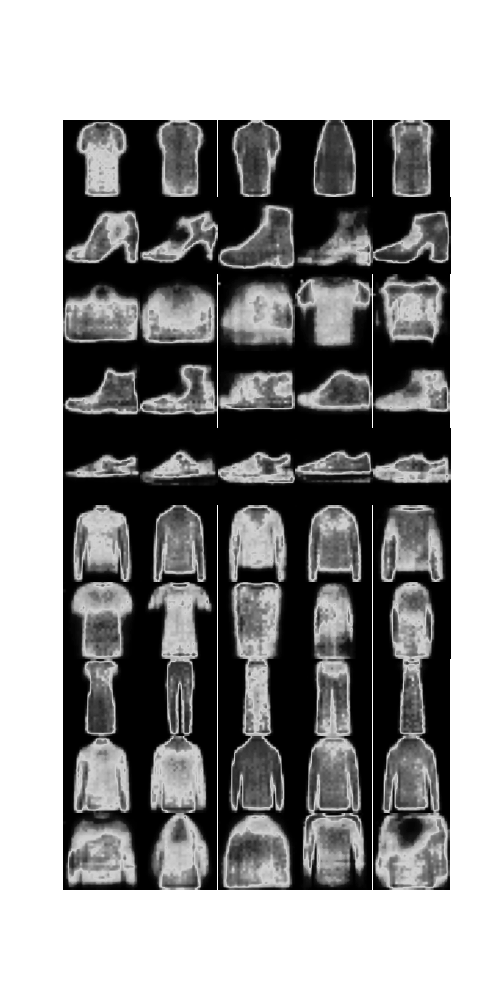}
            \caption{$L_1=L_2=2.5$}
            \label{fig:sam25}
    \end{subfigure}
        \begin{subfigure}{0.2\textwidth}
            \includegraphics[width=\textwidth]{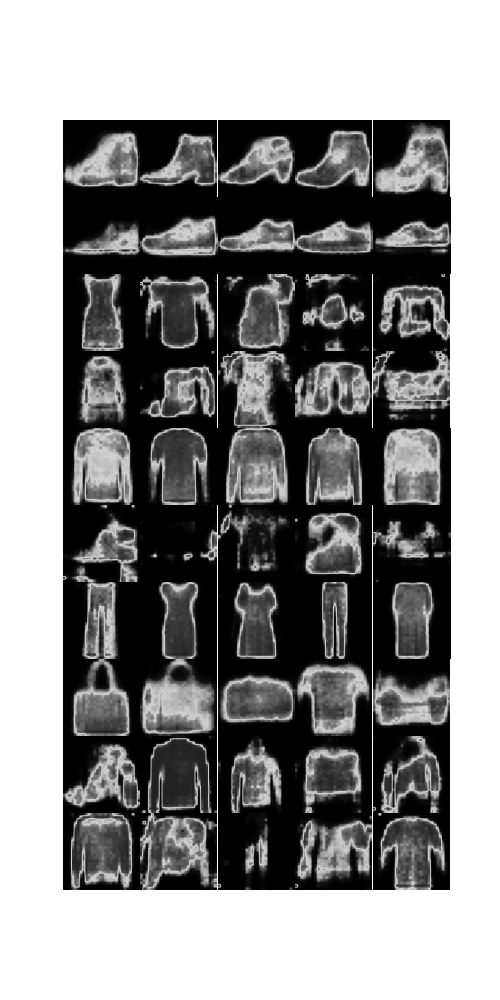}
            \caption{$L_1=L_2=3.5$}
            \label{fig:sam35}
    \end{subfigure}
        \begin{subfigure}{0.2\textwidth}
            \includegraphics[width=\textwidth]{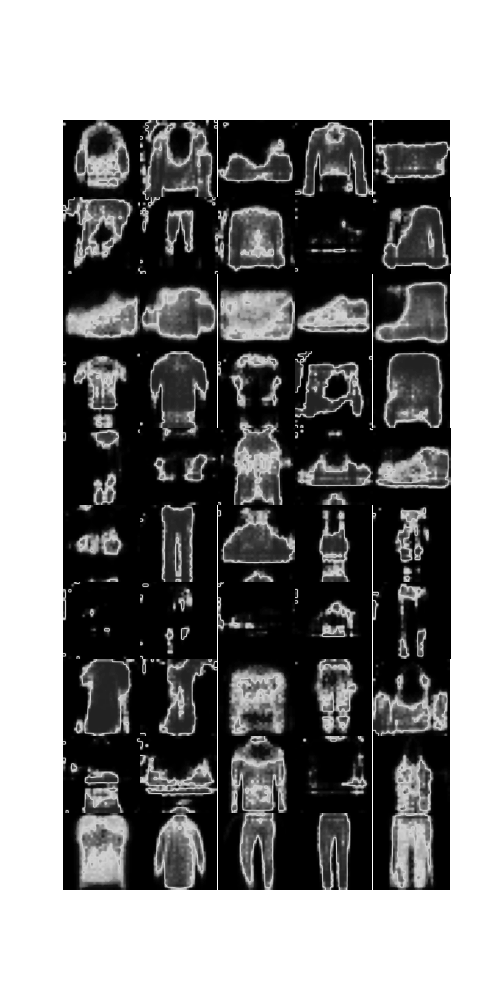}
            \caption{$L_1=L_2=4.5$}
            \label{fig:sam45}
    \end{subfigure}
        \begin{subfigure}{0.2\textwidth}
            \includegraphics[width=\textwidth]{out_fashionmnist_idvae0.0005_5.0_5.0_r10sample_summary.png}
            \caption{$L_1=L_2=5.0$}
            \label{fig:sam50}
    \end{subfigure}
\caption{Samples of Fashion-MNIST data generated with different inverse Lipschitz parameters of IL-LIDMVAE. The number of classes was set to $10$, and all distributions were Gaussian. Each row corresponds to a different category. With $L_1=L_2=1.5$, we obtain the ten true distinct classes with varied images.}
\label{fig:apsamfash}
\vskip -0.2in
\end{figure}
While the reconstruction performance does not change a lot, we can observe that the sampling quality declines with too high values of $L_1$ and $L_2$. $L_1=L_2=1.5$ achieves the best loss and is the only one that learned the ten true distinct classes.
\subsubsection{Text}
We show in Table \ref{tab:apsyn} additional data with different Lipschitz constants.
\begin{table}[t]
\caption{Results for synthetic text data. $\beta$ was set to 0.2 for $\beta$-VAE. The column entitled $L$ refers to the inverse Lipschitz constant of $f_\theta^{(1)}$ and $f_\theta^{(2)}$. NLL stands for negative log-likelihood, MI for mutual information and AU for active units.}
\label{tab:apsyn}
\vskip 0.15in
\begin{center}
\begin{small}
\begin{sc}
\begin{tabular}{lccccc}
\toprule
Model & $L$ & NLL & KL & MI & AU\\
\midrule
VAE   & - & 42.56 & 0.01 & 0.0 & 0.0  \\
$\beta$-VAE ($\beta=0.2$) & - & 42.34 & 0.08 & 0.0 & 0.0 \\
Lagging VAE & - & 45.44 & 2.13 & 1.0 & 1.0\\
LIDVAE & 0    & 56.67 & 0.24 & 0.2 & 0.8 \\
IL-LIDVAE &0.2  & \textbf{40.39} & 0.32 &0.3 & 1.0 \\
IL-LIDVAE &0.4 & 40.48 & 0.39 & 0.3 & 1.0 \\
IL-LIDVAE &0.5 & 40.48 & 0.60 & 0.5 & 1.0 \\
IL-LIDVAE &0.6  & 41.65 & 0.85 & 0.6 & 1.0 \\
IL-LIDVAE &1.0  & 44.06 & 2.45 & 0.8 & 1.0 \\
IL-LIDVAE &1.5  & 44.50 & 3.80 & 0.9 & 1.0 \\
IL-LIDVAE &5.0  & 52.34 & 8.13 & 1.3 & 1.0 \\
\ \ +annealing & - & 39.6 & 0.38 & 0.1 & 1.0 \\
\bottomrule
\end{tabular}
\end{sc}
\end{small}
\end{center}
\vskip -0.1in
\end{table}

%% file: apc.tex
In this appendix, we briefly describe the Log-Sobolev inequality, mentioned in the main paper, and some well-known properties as well. We only treat distributions absolutely continuous with respect to the Lebesgue measure for simplicity.
\begin{definition}
Distribution $\nu$ satisfies the \em{Log-Sobolev inequality} (LSI) with a constant $\alpha$ if for all probability density functions $\rho$ absolutely continuous with respect to $\nu$, the following holds:
\begin{align*}
  D(\rho\mid\mid\nu)\le \frac{1}{2\alpha}F(\rho\mid\mid\nu),
\end{align*}
where $D(\rho\mid\mid\nu)=\E_{\rho}\left[ \log{\frac{\rho}{\nu}}\right]$ is the KL-divergence of $\rho$ with respect to $\nu$, and $F(\rho\mid\mid\nu)=\E_{\rho}\left[ \left\|\nabla\log{\frac{\rho}{\nu}}\right\|^2\right]$ is the relative Fisher information divergence of $\rho$ with respect to $\nu$.
\end{definition}
The Gaussian distribution satisfies LSI as implied by the following proposition.
\begin{proposition}[\citet{BE2006}]
Suppose $q\propto \e^{-f}$ is a probability density, where $f:\mathbb{R}^d\to \mathbb{R}$ is a smooth function. If there exists a constant $c>0$ such that $\nabla^2 f\succeq cI_d$, then $q(z)dz$ satisfies LSI with constant $c$.
\end{proposition}
Therefore, if we are only using Gaussian distributions in the VAE, then posterior collapse in terms of the relative Fisher information divergence results in posterior collapse in terms of KL-divergence.
\par The following statements show that LSI is robust under bounded perturbations and Lipschitz mappings.
\begin{proposition}[\citet{HS1986}]
Suppose $q$ is a probability density that satisfies LSI with constant $\alpha$. For any bounded function $B:\mathbb{R}^d\to\mathbb{R}$, $q_B\propto \e^{B}q$ satisfies LSI with constant $\alpha\e^{-4\|B\|_\infty}$.
\end{proposition}
\begin{proposition}[\citet{VW2019}]
Suppose $q$ is a probability density that satisfies LSI with constant $\alpha$. If $H:\mathbb{R}^d\to\mathbb{R}^d$ is a differentiable $L$-Lipschitz mapping, then the distribution of $H(z)$ with $z\sim q(z)$ satisfies LSI with constant $\alpha/L^2$.
\end{proposition}